\documentclass{article}

\usepackage{fullpage}

\usepackage[utf8]{inputenc}
\usepackage[T1]{fontenc}

\usepackage{booktabs}
\usepackage{array}
\usepackage{amsfonts}
\usepackage{nicefrac}
\usepackage{microtype}
\usepackage{xcolor}
\usepackage[ruled]{algorithm2e}
\usepackage{tikz}
\usepackage{amsmath}
\usepackage{amsthm}
\usepackage{amssymb,mathtools}
\usepackage{enumitem}
\usepackage{natbib}

\usepackage[hypertexnames=false,bookmarks=false]{hyperref}
\usepackage{cleveref}
\usepackage{autonum}

\newtheorem{theorem}{Theorem}
\newtheorem{corollary}{Corollary}
\newtheorem{lemma}{Lemma}
\newtheorem{proposition}{Proposition}
\newtheorem{definition}{Definition}
\newtheorem{remark}{Remark}
\newtheorem{example}{Example}

\newcommand{\Cal}{\operatorname{CalErr}}
\newcommand{\MCerr}{\operatorname{MCerr}}

\newcommand{\cA}{\mathcal{A}}
\newcommand{\cD}{\mathcal{D}}
\newcommand{\cE}{\mathcal{E}}
\newcommand{\cF}{\mathcal{F}}
\newcommand{\cG}{\mathcal{G}}
\newcommand{\cH}{\mathcal{H}}
\newcommand{\cJ}{\mathcal{J}}
\newcommand{\cP}{\mathcal{P}}
\newcommand{\cV}{\mathcal{V}}

\title{Instance-Adaptive Online Multicalibration}
\date{May 22, 2026}

\newcommand{\affmark}[1]{\textsuperscript{#1}}

\author{%
  \parbox{\textwidth}{\centering
    Zhiming Huang\affmark{1},
    Jamie Morgenstern\affmark{1,3},
    Aaron Roth\affmark{2},
    and Claire Jie Zhang\affmark{1}\\[0.6em]
    \affmark{1}Paul G. Allen School of Computer Science and Engineering,
    University of Washington\\[0.3em]
    \affmark{2}Department of Computer and Information Sciences,
    University of Pennsylvania\\[0.3em]
    \affmark{3}Amazon
  }
}

\begin{document}

\maketitle

\begin{abstract}
We study online multicalibration beyond the worst-case. We give a single, efficient
algorithm that dynamically interpolates between benign and worst-case sequences by adaptively
refining a dyadic grid of prediction values. Its error is controlled by the growth of the
refinement tree. Our analysis recovers the known $\widetilde O(T^{2/3})$ worst-case-optimal rate
for online multicalibration, while automatically adapting to easier instances: in the marginal
stochastic setting it obtains a rate of $\widetilde O(\sqrt T)$, and for piecewise-stationary means
with $J$ segments its rate is $\widetilde O(\sqrt{JT})$. More generally, the rate depends on how
well the predictable means can be approximated by simple contextual scores whose thresholds can
be represented by the group family.  We also show that the threshold-complexity dependence is tight up to
logarithmic factors.
\end{abstract}

\section{Introduction}\label{sec:intro}
Calibration is a basic statistical property that we want for probabilistic predictions to be ``trustworthy''. At a surface level, calibration asks that predictions ``mean what they say'' in the sense that they should be unbiased predictors of the outcome, even conditional on the value of the prediction itself.  For example, a calibrated weather forecaster should have the property that, amongst all days the forecaster predicts a 20\% chance of rain, 20\% of those days actually experience rain~\citep{dawid1982well}.
Calibrated predictions are viewed as ``trustworthy'' because of how they mediate the interface between prediction and decision making: selecting an action that is a best response to calibrated forecasts is an optimal decision policy among all policies that map forecasts to actions, simultaneously for all sets of actions and utility functions \citep{FV98,kleinberg2023u,noarov2023high,roth2022uncertain,kiyani2025robust}. 

Remarkably, as first shown by \citet{foster1998asymptotic}, it is possible to produce forecasts that are asymptotically calibrated even in sequential \emph{adversarial} prediction settings, in which there is no underlying distribution, and outcomes can be chosen adaptively by an adversary. The rate at which algorithms can achieve calibration depends on the environment. In the standard batch/statistical learning setting, calibration can be obtained via marginal mean estimation: given $T$ samples, simple algorithms can produce predictors whose empirical calibration error, without the conventional $1/T$ normalization, scales as $\Theta(\sqrt{T})$ \citep{shabat2020sample, gupta2021distribution}. Calibration in the online adversarial setting, on the other hand, is a harder problem with the best known lower bounds of $\Omega(T^{0.543})$ separating it from the stochastic setting 
\citep{qiao2021stronger,dagan2025breaking}.

Regardless of the rate at which one achieves it,
calibration alone guarantees surprisingly little: it is a marginal property, demanding unbiased predictions only on average across the full sequence. We instead consider guarantees which ensure predictions are calibrated simultaneously on each of many subsequences that can be defined by the data \citep{dawid1985calibration,lehrer2001any,sandroni2003calibration}. The modern formulation of this idea---multicalibration \citep{hebert2018multicalibration}---is achievable at a rate of $\tilde O(T^{2/3})$ in online adversarial settings \citep{noarov2023high,ghuge2025improved}, and this rate is tight in both the sequential adversarial and batch settings \citep{collina2026batch,collina2026optimal}. The online lower bounds, however---$\Omega(T^{0.543})$ for marginal calibration \citep{dagan2025breaking} and $\tilde \Omega(T^{2/3})$ for multicalibration \citep{collina2026optimal}---hold only in the worst case, leaving open whether typical instances admit substantially better rates. We ask:
\begin{quote}
  \emph{Are there sequential multicalibration algorithms that go beyond the worst case, achieving instance-dependent bounds that interpolate between $O(\sqrt{T})$ rates on easy instances and the $O(T^{2/3})$ rates that are optimal in the worst case?}
\end{quote}

We show that the answer is \emph{yes}. There is a single, efficient algorithm that obtains 
(multi)calibration rates adaptive to the instance. It obtains $\tilde O(\sqrt{T})$ calibration rates on stochastic instances without groups (which is impossible for worst case adversaries), recovers the minimax-optimal $\widetilde O(T^{2/3})$ worst-case rate for polynomial-sized group families, and more generally gives a continuum of bounds depending on how well the predictable means can be approximated by simple structured scores. For example, without groups, if the sequence of outcomes is produced by a piecewise-stationary environment in which the outcome $y_t$ is drawn from a distribution with mean $q_t$ which changes at most $J-1$ times across the $T$ rounds, we show that our algorithm obtains calibration error scaling as $\tilde O(\sqrt{JT})$ --- recovering the claimed stochastic bound as the special case in which $J = 1$. We also obtain nontrivial rates when the means are not exactly stationary on any long segment: if $C_{\rm stat}(q)=\inf_{c\in[0,1]}\sum_t|q_t-c|$, then the same  algorithm obtains $\widetilde O(\sqrt T+(TC_{\rm stat}(q))^{1/3})$.

These marginal calibration guarantees are corollaries of our more general theorem for multicalibration over groups. The relevant difficulty of a multicalibration instance is governed by how well one can predict label means from the context in a way that the group family can represent. Concretely, we compare the predictable means to simple contextual scores. A score is useful to us when it takes only a few values and its thresholds can be represented cheaply using the group functions. The final rate also pays for the residual error $\sum_{t\in S}|\mu_t-f_S(x_t)|$ left after approximating the predictable means by such scores on contiguous blocks. Thus label means may change at every round and still define an easy instance, provided those changes are predictable from context and their thresholds are represented by the group family.

Like many online calibration algorithms, our algorithm maintains a discrete grid of prediction values and computes the minimax optimal strategy in a game defined by a weighting of the worst-case next-round bias. Unlike many standard algorithms, ours does not maintain a fixed set of prediction values. Instead, it adaptively refines the grid over time, refining each interval
once the algorithm has played it many times.
A run of the algorithm produces an adaptively ``grown'' tree of prediction values. We first bound calibration error by accounting for the intervals that become active at each depth of this tree. The rest of the analysis upper bounds the number of active intervals at each depth in terms of the structure of the predictable means: intervals far from a good contextual approximation create detectable bias and therefore cannot be played often.

\subsection{Related Work}\label{appsec:related_work}

There is a large literature on online multicalibration and closely related  guarantees. The modern formulation of multicalibration is due to \cite{hebert2018multicalibration}, but the first online multicalibration algorithms predate this formulation \citep{lehrer2001any,sandroni2003calibration,foster2006calibration,foster2011complexity}. 
\citet{gupta2022online} gave the first efficient online multicalibration algorithm with rates scaling logarithmically with the number of groups, as well as generalizations to related guarantees (roughly speaking multicalibration for quantiles and moments).  \citet{lee2022online} formulated
online minimax multiobjective optimization and applied it to a variety of problems including 
multicalibration-style objectives, following the game theoretic tradition of deriving online calibration algorithms \citep{hart2025calibrated,fudenberg1999easier,foster1999proof,foster2018smooth}. \cite{haghtalab2023unifying} give an alternative unifying treatment of multicalibration algorithms through the lens of solving zero sum games with online learning algorithms.  \citet{farina2026efficient} interpret the kind of minimax step used in these papers as solving an expected variational inequality, and bring fast solvers to bear on it. \citet{garg2024oracle,ghuge2025improved,hu2025efficient,farina2026efficient} give oracle-efficient online
multicalibration guarantees for rich benchmark classes. Their emphasis is
oracle efficiency for large/continuous classes, whereas our algorithm is efficient only for explicitly represented
finite group families and gives instance-adaptive rates. More recently, \citet{hu2025efficient}
studied efficient swap multicalibration for elicitable properties, giving online multicalibration algorithms for properties beyond means (cf. \cite{noarov2023statistical}). The minimax sample complexity of online multicalibration was recently settled by \cite{collina2026batch,collina2026optimal} in both the batch and online settings, showing that the worst-case upper bound given by \cite{noarov2023high} was tight.

Bandit learning is another setting where there is a gap between stochastic and adversarial worst-case rates. There is a similarly motivated line of work that gives single algorithms obtaining the ``best of both worlds'' guarantees --- i.e. optimal stochastic rates on stochastic instances without giving up on worst case rates on adversarial instances \citep{bubeck2012best, de2014follow}. Similarly a branch of the online convex optimization literature studies more fine-grained regret bounds that adapt to properties of the realized loss functions, that can give better-than-worst-case $o(\sqrt{T})$ regret bounds on well behaved instances (see e.g. \citep{chiang2012online, chen2021impossible}). 

The idea of parameterizing the complexity of an adversary by how many times it changes strategies also appears in the literature on ``tracking the best expert'' \citep{herbster1998tracking,bousquet2002tracking} and adaptive regret \citep{hazan2007adaptive}, which give regret bounds not just to the best expert in hindsight, but to the best sequence of experts that change a small number of times, or that have a small smoother measure of drift. 

Our adaptive refinement of prediction values is reminiscent of ``zooming algorithms'' for the Lipschitz bandits problem in which algorithms iteratively refine attention to more promising regions of action space \citep{kleinberg2008multi,slivkins2011contextual,podimata2021adaptive,krishnamurthy2020contextual}. From this literature, \citet{podimata2021adaptive,krishnamurthy2020contextual} are the most closely related in that they give instance-dependent bounds. 

Recently, \citet{hu2026near} gave a swap regret algorithm for Lipschitz convex losses (where the action space is continuous) by using multiple scales of action discretizations. This is conceptually related to our approach; both approaches are designed to circumvent bottlenecks that result from algorithms tied to fixed discretizations of continuous action spaces. Despite the conceptual similarity, the techniques and results differ significantly; their algorithm simultaneously uses multiple discretization scales, and they prove worst-case regret bounds for different swap-regret objectives. Our algorithm adaptively refines predictions in a tree, and we prove beyond-worst-case-analysis style calibration bounds that take advantage of trajectory-specific structure. 

Concurrently, \citet{liu2026adaptive} give an alternative algorithm for obtaining instance-adaptive marginal calibration bounds. The first versions of our papers gave incomparable results in the marginal calibration setting. In the first version of our paper, our most interpretable corollary in the marginal calibration setting was the bound we state as Corollary \ref{cor:marginal-l1} in the current version, which obtains calibration at the rate of $\min(T^{2/3}, \sqrt{JT})$ for stochastic instances whose means change J times, whereas the main result of \citet{liu2026adaptive} was a bound of $\sqrt{T} + (TC_{\rm stat}(q))^{1/3}$, where $C_{\rm stat}(q)$ is a measure of the average distance of the label mean from its median across time. Neither bound implies the other. However after seeing each other's papers, we realized that both methods could be used to give both bounds --- we now also state the $\sqrt{T} + (TC_{\rm stat}(q))^{1/3}$ bound as corollary \ref{cor:marginal-special}, matching a theorem first given by \cite{liu2026adaptive}. Beyond the overlapping results in the marginal calibration setting, the two papers generalize in different directions. The marginal calibration results we state are corollaries of the more general adaptive multicalibration bound we give for our algorithm, but we restrict our analysis to the ECE metric. \cite{liu2026adaptive} give an algorithm only for marginal calibration, but give results for a broader collection of calibration metrics.

\section{Online Multicalibration and Dynamic Bins}
\label{sec:setup}

\subsection{Protocol and Error Notion}
We consider the online multicalibration problem where in each round $t \in [T]$, the learner observes a context $x_t \in \mathcal{X}$, and chooses a mixed forecast $\pi_t$ (a distribution over finitely many candidate prediction values in $[0,1]$), and then privately samples a realized prediction $p_t$ from $\pi_t$, and then observes an outcome $y_t \in [0, 1]$ revealed by nature.

Formally, at round $t\in[T]$:
\begin{enumerate}[leftmargin=18pt]
\item the learner observes a context $x_t\in\mathcal X$;
\item the learner outputs a mixed forecast $\pi_t$, i.e. a distribution over finitely many
prediction values in $[0,1]$;
\item a realized prediction $p_t\in[0,1]$ is sampled from $\pi_t$ using fresh private randomization
that is not revealed to Nature before the outcome is generated;
\item Nature reveals an outcome $y_t\in[0,1]$.
\end{enumerate}

We assume that $y_t$ may depend on the past, on the current context $x_t$, and on the mixed
forecast $\pi_t$, but not on the fresh randomization used to sample $p_t$. That is, Nature may react to $\pi_t$, but not to
the private coin flip that turns $\pi_t$ into the realized prediction $p_t$.

We work with the following standard two-stage filtration. The sigma-field $\cF_{t-1}$ contains the
realized transcript through the end of round $t-1$. Define the pre-outcome sigma-field
\[
\cH_{t-1}:=\sigma(\cF_{t-1},x_t,\pi_t),
\]
and the end-of-round history
\[
\cF_t:=\sigma(\cH_{t-1},p_t,y_t).
\]
Thus $\cH_{t-1}$ contains the information available before the outcome is generated, while
$\cF_t$ contains the full transcript through round $t$. Conditional on $\cH_{t-1}$, the fresh
randomization used to sample $p_t$ has law $\pi_t$ and is independent of $y_t$. Define the
next pre-outcome sigma-field, for $t<T$, by
\[
\cH_t:=\sigma(\cF_t,x_{t+1},\pi_{t+1}),
\]
so that
\[
\cH_{t-1}\subseteq \cF_t\subseteq \cH_t
\]
for every $t<T$. For notational convenience set $\cH_T:=\cF_T$, so that the same interleaving
convention holds for all $t\in[T]$. Define the
predictable conditional mean
$\mu_t:=\mathbb E[y_t\mid \cH_{t-1}] \in [0,1]$.
This is just notation for the conditional mean of the realized outcome sequence. In particular, if
Nature chooses $y_t$ deterministically from $\cH_{t-1}$, then $\mu_t=y_t$.



Let $\cG$ be a finite family of binary groups $g:\mathcal X\to\{0,1\}$. If the all-ones
group does not belong to the class we add it and write $\bar{\cG}:=\cG\cup\{\mathbf 1\}$.

\begin{definition}[Calibration error]
\label{def:empirical-mcerr}
For a group collection $\cG$ and realized transcript,
the  multicalibration error is
\[
\MCerr_{\cG}(T)
:=
\max_{g\in\cG}
\sum_{v\in\{p_1,\dots,p_T\}}
\left|
\sum_{t:p_t=v} g(x_t)(y_t-v)
\right|.
\]
The marginal calibration error is:
\[
\Cal(T)
:=
\sum_{v\in\{p_1,\dots,p_T\}}
\left|
\sum_{t:p_t=v}(y_t-v)
\right|.
\]
\end{definition}

\subsection{Dynamic bins}
Now we introduce the core structure of our algorithm: a dynamic, multiscale partition of the prediction space $[0, 1]$.
Let $d_{\max}:=\left\lfloor \tfrac12\log_2 T \right\rfloor+1$. For each depth $d=0,\dots,d_{\max}$, define the dyadic grid
\[
\cD_d
:=
\Bigl\{[k2^{-d},(k+1)2^{-d}) : k=0,\dots,2^d-2\Bigr\}
\cup
\Bigl\{[1-2^{-d},1]\Bigr\}.
\]
Use $\mathcal D$ to denote the union of intervals\footnote{We use intervals and bins interchangeably throughout the paper.} at all depths, $\mathcal D:=\bigcup_{d=0}^{d_{\max}} \mathcal D_d$. For each interval $I\in\cD$, let $r_I$ denote its midpoint
and let $w_I = \sup_{x, y \in I}|x - y|$ denote its width. For intervals at depth $d$, we sometimes denote their width as $w_d$. The learner maintains an active dyadic partition $B_t\subseteq \cD$ of $[0,1]$, starting from
$B_1=\{[0,1]\}$. We call an interval $I$ ``active" 
for rounds $t$ where $I\in B_t$. Let
$\beta(I):=\min\{t\in[T]: I\in B_t\}$, $\delta(I):=\max\{t\in[T]: I\in B_t\}$.
We refer to $\cV_T:=\bigcup_{t=1}^T B_t$
as the set of intervals that are ever active. We denote by $\cV_d$  the set of depth-$d$ intervals that are ever active, $\cV_d=\cV_T\cap\cD_d$, and $m_d$ is the cardinality of $\cV_d$.
In round $t$, we denote the  mixed forecast as
$\pi_t\in\Delta(B_t)$, where $\pi_{t,I}$ denotes the probability assigned to the midpoint $r_I$ of interval $I$.

For a set of rounds $S\subseteq[T]$, define the (expected) total play of interval $I$ on $S$ by $N_S(I):=\sum_{t\in S}\pi_{t,I}$,
with the convention $\pi_{t,I}=0$ when $I\notin B_t$. 
We denote the full-horizon total play as $N(I):=N_{[T]}(I)$. Let $L:=\log(eT|\bar{\cG}|)$. Each interval $I$ of depth strictly smaller than $d_{\max}$ is split at the end of the first round $t$ where its running total play satisfies 
$N_{[t]}(I)\ge \frac{L}{w_I^2}$.
When $I$ splits, it is permanently removed from active partition $B_t$, meaning its final total play $N(I)$ remains fixed at $N_{[t]}(I)$. Its two dyadic children are added to the active partition in the next rounds. Each child interval is of width $w_I/2$ and initialized with total play value $0$.
Our algorithm has contiguous active lifetimes for every interval,
which we denote as $A(I):=[\beta(I),\delta(I)]$, and $\pi_{t,I}=0$ for all $t\notin A(I)$.

\begin{figure}[t]
\centering
\begin{tikzpicture}[
    x=1cm,
    y=1cm,
    activeleaf/.style={draw=blue!60!black, fill=blue!14, thick},
    splitnode/.style={circle, draw=black, fill=white, inner sep=1.6pt},
    leafnode/.style={circle, draw=blue!60!black, fill=blue!14, inner sep=1.6pt},
    ticklab/.style={font=\scriptsize},
    panellab/.style={font=\small\bfseries}
]
\node[panellab] at (3.0,3.0) {Active partition $B_t$};
\foreach \a/\b in {0/1.5,1.5/3,3/4.5,4.5/5.25,5.25/6} {
  \filldraw[activeleaf] (\a,2.15) rectangle (\b,2.7);
}
\draw[thick] (0,2.15) rectangle (6,2.7);
\foreach \x/\lbl in {0/$0$,1.5/$\frac14$,3/$\frac12$,4.5/$\frac34$,5.25/$\frac78$,6/$1$} {
  \draw[gray!70] (\x,2.08)--(\x,2.78);
  \node[ticklab, anchor=north] at (\x,2.05) {\lbl};
}
\foreach \x in {0.75,2.25,3.75,4.875,5.625} {
  \fill[blue!70!black] (\x,2.425) circle (1.2pt);
}
\node[ticklab, align=center] at (3.0,1.02) {the learner predicts a mixture over\\the leaf midpoints $\{r_I:I\in B_t\}$};

\node[font=\Large] at (7.1,2.42) {$\Longleftrightarrow$};

\node[panellab] at (10.7,3.0) {Corresponding dyadic tree};
\coordinate (root) at (10.7,2.45);
\coordinate (L) at (9.2,1.55);
\coordinate (R) at (12.2,1.55);
\coordinate (LL) at (8.45,0.7);
\coordinate (LR) at (9.95,0.7);
\coordinate (RL) at (11.45,0.7);
\coordinate (RR) at (12.95,0.7);
\coordinate (RRL) at (12.55,-0.15);
\coordinate (RRR) at (13.35,-0.15);
\draw[thick] (root)--(L) (root)--(R);
\draw[thick] (L)--(LL) (L)--(LR);
\draw[thick] (R)--(RL) (R)--(RR);
\draw[thick] (RR)--(RRL) (RR)--(RRR);
\node[splitnode] at (root) {};
\node[splitnode] at (L) {};
\node[splitnode] at (R) {};
\node[splitnode] at (RR) {};
\node[leafnode] at (LL) {};
\node[leafnode] at (LR) {};
\node[leafnode] at (RL) {};
\node[leafnode] at (RRL) {};
\node[leafnode] at (RRR) {};
\node[ticklab, anchor=north] at (LL) {$[0,\frac14)$};
\node[ticklab, anchor=north] at (LR) {$[\frac14,\frac12)$};
\node[ticklab, anchor=north] at (RL) {$[\frac12,\frac34)$};
\node[ticklab, anchor=north] at (RRL) {$[\frac34,\frac78)$};
\node[ticklab, anchor=north] at (RRR) {$[\frac78,1]$};
\node[ticklab, align=center] at (10.7,-0.98) {internal nodes have already split;\\leaves are the current active bins};
\end{tikzpicture}
\caption{A typical state of the dynamic-bin data structure. The learner predicts using the midpoints
of the current leaves, and an interval is refined only after the total mass assigned to it reaches the threshold
$L/w_I^2$.}
\label{fig:partition-tree}
\end{figure}
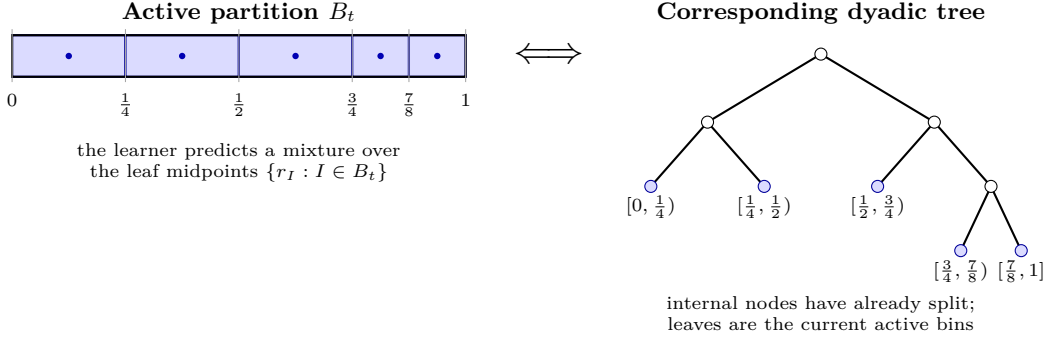

\paragraph{Why this split threshold?}
The threshold $N(I)\ge L/w_I^2$ is the scale at which further refinement becomes worthwhile. An interval of width $w_I$ contributes deterministic discretization error on the order of $N(I)\cdot w_I$,
while the corresponding online-learning term scales like $\sqrt{N(I)L}$. The crossover point $N(I)\asymp \frac{L}{w_I^2}$
is where these two terms are equal. The algorithm therefore splits an interval when the discretization error has become the dominant source of
bias and not before.

\subsection{The general algorithm}


For every group $g\in\bar{\cG}$, interval $I$, and sign $\sigma\in\{+,-\}$, define the 
group-weighted signed bias terms
\begin{align}
\phi^+_{g,I}(\pi_t,y_t)&:=g(x_t)\pi_{t,I}(y_t-r_I-w_I),
&
\phi^-_{g,I}(\pi_t,y_t)&:=g(x_t)\pi_{t,I}(r_I-y_t-w_I).
\end{align}
The extra slack term
$w_I$  is a discretization buffer: deviations of size at most $w_I$ are below the resolution
of interval $I$, so they should not count as evidence that $r_I$ is too large or too small.

We combine these  weighted bias terms using a  sleeping-experts algorithm, instantiated with 
confidence-rated AdaNormalHedge~\citep{luo2015achieving}.
Each interval $I$ has many experts corresponding to it, indexed by a start time $s$, group $g$, and sign $\sigma$. A particular expert $(s, g, I, \sigma)$ is awake in round $t$ if and only if $s \leq t$ (the expert's start time $s$ has passed) and $I \in B_t$ (the corresponding interval is active).
We include start times $s$ so that we can control weighted bias on every contiguous sub-interval of rounds, not just marginally.

\paragraph{The sleeping-experts interface.}
At round $t$, the wrapper maintains a nonnegative weight $\omega_{t,e}$ for each expert $e$ that
is awake in that round. These weights are normalized so that
$\sum_{e:\text{ awake at }t}\omega_{t,e}=1$.
For each group $g$ and interval $I\in B_t$, define the aggregate sign weights assigned to positive and negative copies of $(g, I)$:
\[
\lambda^+_{t,g,I}
:=
\sum_{s\le t}\omega_{t,(s,g,I,+)},
\qquad
\lambda^-_{t,g,I}
:=
\sum_{s\le t}\omega_{t,(s,g,I,-)}.
\]
We then track interval-level aggregations of these weights
\[
a_{t,I}
:=
\sum_{g\in\bar{\cG}} g(x_t)\bigl(\lambda^+_{t,g,I}-\lambda^-_{t,g,I}\bigr),
\qquad
b_{t,I}
:=
\sum_{g\in\bar{\cG}} g(x_t)\bigl(\lambda^+_{t,g,I}+\lambda^-_{t,g,I}\bigr).
\]
In words, 
 $a_{t,I}$ records the net signed weight favoring outcomes above rather
than below $r_I$, while $b_{t,I}$ records the total weight assigned to both of the corresponding
one-sided bias terms for $I$. By construction, $|a_{t,I}|\le b_{t,I}$.

The wrapper's aggregate one-step bias is the weighted average of these expert-level biases:
\[
\widehat \phi_t
:=
\sum_{I\in B_t}\sum_{g\in\bar{\cG}}
\Bigl(
\lambda^+_{t,g,I}\phi^+_{g,I}(\pi_t,y_t)
\;+\;
\lambda^-_{t,g,I}\phi^-_{g,I}(\pi_t,y_t)
\Bigr),
\]
and regrouping by interval gives the identity
$
\widehat \phi_t
=
\sum_{I\in B_t}
\pi_{t,I}\bigl(a_{t,I}(y_t-r_I)-b_{t,I}w_I\bigr).$

A forecast $\pi_t$ is then chosen so that no convex combination of the currently active
one-sided weighted bias terms has positive aggregate one-step bias, regardless of which outcome $y \in[0,1]$ Nature
reveals. In particular, the learner chooses $\pi_t\in\Delta(B_t)$ satisfying
\begin{equation}\label{eq:weighted-feas}
\sum_{I\in B_t} \pi_{t,I}\bigl(a_{t,I}(y-r_I)-b_{t,I}w_I\bigr) \le 0
\qquad
\text{for every }y\in[0,1].
\end{equation}
\begin{lemma}[Weighted feasibility]\label{lem:weighted-feasibility}
Let $B$ be any finite partition of $[0,1]$ into intervals $I$ with midpoints $r_I$ and widths $w_I$. Let
$\{a_I,b_I\}_{I\in B}$ satisfy $|a_I|\le b_I$ for every $I\in B$. Then there exists
$\pi\in\Delta(B)$ such that
\begin{equation}
\sum_{I\in B}\pi_I(a_I(y-r_I)-b_Iw_I)\le 0
\qquad
\text{for every }y\in[0,1].    
\end{equation}
\end{lemma}
The feasibility of this set of constraints, and the existence of such a mixed forecast, follow from a weighted feasibility argument based on Sion's minimax theorem (see Appendix~\ref{pf:lem:weighted}). Because the expression $\sum_{I \in B_t} \pi_{t, I}\left(a_{t, I}\left(y-r_I\right)-b_{t, I} w_I\right)$ is affine in $y$, Equation~\eqref{eq:weighted-feas} holds for all $y \in[0,1]$ if and only if it holds at the endpoints $y=0$ and $y=1$. Consequently, finding the mixed forecast $\pi_t$ computationally reduces to solving a standard linear program over the simplex with exactly two constraints.


The learner then samples $p_t\sim \pi_t$, observes $y_t$,  provides the normalized losses derived from $\phi_{g,I}^{\pm}(\pi_t,y_t)$ to the sleeping-experts wrapper,  and splits any interval
satisfying the threshold $L/w_I^2$. The complete round-by-round procedure is summarized in
Algorithm~\ref{alg:mc-dynamic-bins}.

\begin{algorithm}[H]
\SetAlgoNoLine
\KwIn{horizon $T$, group family $\cG$}
Initialize $B_1=\{[0,1]\}$ and $N(I)=0$ for the root interval\;
Instantiate a sleeping-experts wrapper using confidence-rated AdaNormalHedge on the tuples $(s,g,(I,\pm))$\;
\For{$t=1,\dots,T$}{
  Observe the context $x_t$\;
  Aggregate the weights of the awake experts into coefficients $a_{t,I},b_{t,I}$\;
  Find $\pi_t\in\Delta(B_t)$ satisfying (\ref{eq:weighted-feas})\;
  Sample a realized prediction $p_t=r_I$ with probability $\pi_{t,I}$\;
  Observe the outcome $y_t\in[0,1]$\;
  Update the losses of the awake experts using the group-weighted bias terms $\phi^\pm_{g,I}(\pi_t,y_t)$\;
  Update $N(I)\leftarrow N(I)+\pi_{t,I}$ for every $I\in B_t$\;
  Split every $I\in B_t$ of depth $d<d_{\max}$ with $N_{[t]}(I)\ge L/w_I^2$ into its two children, each initialized with counter value $0$\;
}
\caption{Dynamic bins for online multicalibration}
\label{alg:mc-dynamic-bins}
\end{algorithm}

\section{Why the Algorithm Works: Marginal Calibration Intuition}
\label{sec:intuition}

Before defining the instance-complexity measures formally, we explain the proof idea in the
simplest special case: the context space is a singleton and the only group is the all-ones group.
In this case multicalibration is just marginal calibration. The full proof in
Section~\ref{sec:proofs} follows the same template but in greater generality.

\paragraph{Step 1: Regret controls per-interval bias.}
Fix an active interval $I$ with midpoint $r_I$ and width $w_I$. The sleeping-experts
wrapper tracks whether predictions from this interval are accumulating positive or
negative bias. The weighted feasibility step chooses $\pi_t$ so that the aggregate
one-step bias is nonpositive for every possible next outcome. As a result, over any
contiguous block $S$ inside the lifetime of $I$,
\[
\left|\sum_{t\in S}\pi_{t,I}(y_t-r_I)\right|
\lesssim
N_S(I)w_I+\sqrt{N_S(I)\log T}+\log T.
\]
The first term is deterministic discretization error: an interval of width $w_I$
cannot distinguish outcomes that differ by less than its own resolution. The
remaining terms are the regret cost of controlling one-sided bias on that interval.
This also explains the split rule. The algorithm refines $I$ when
$N(I)\asymp \log T/w_I^2$, the point where the discretization term $N(I)w_I$ and
the regret term $\sqrt{N(I)\log T}$ are of the same order.

\paragraph{Step 2: Far away intervals accumulate bias quickly}  Write $q_t:=\mu_t$
for the scalar predictable mean at time $t$. 
Now suppose that on a time block $S$, the mean sequence $q_t$ is well approximated by a
constant $c_S$. If a depth-$d$ interval has midpoint $r_I$ far above $c_S$, then
putting probability on that interval tends to produce negative bias; if $r_I$ is
far below $c_S$, it tends to produce positive bias. Step 1 says such one-sided bias
cannot persist, so far-away intervals cannot receive too much mass.

If $q_t$ is only approximately constant, this argument weakens exactly on rounds
where $q_t$ is not close to $c_S$. Errors smaller than the bin width $w_d$ are below
the resolution of a depth-$d$ interval, so they are free in our accounting. The relevant residual on
$S$ at depth $d$ is therefore
\[
\mathsf R_{S,d}(c_S)
:=
\sum_{t\in S}(|q_t-c_S|-w_d)_+.
\]
How frequently a far interval can be played can be bounded by two terms: a square-root deviation
term from regret and concentration, and a residual term measuring how much of
$\mathsf R_{S,d}(c_S)$ is assigned to that interval.

\paragraph{Step 3: Residual error controls how many intervals split.}
A depth-$d$ interval splits only after accumulating total probability mass on the order of
$\log T/w_d^2$. Around the local constant $c_S$, there may be a band of nearby
intervals for which the bias argument is too weak. Farther away, Step 2 gives a
bound on how much probability mass an interval can receive.

The analysis balances these two effects. If we declare a wider band around $c_S$
to be ``near,'' we pay for more nearby intervals. If we declare a narrower band to
be near, we must charge more of the far-away play to residual error. Optimizing this
tradeoff gives a bound on the number of splits at depth $d$ on the order of:
\[
1+\sqrt{\mathsf R_{S,d}(c_S)w_d/\log T}
\]
for a single block $S$. If the block is exactly stationary at scale $w_d$, then
the residual is zero and the block contributes only a constant number of splits at
that depth. Summing these scale-by-scale split bounds is what ultimately produces
the cube-root residual term in the main theorem.

\paragraph{Step 4: Depth-wise active intervals control calibration error.}
The tree contains intervals at many resolutions. A coarse interval has large
discretization error but can only appear in small numbers; a fine interval has small
width but can be expensive to control if many such intervals are active. The proof
therefore accounts for calibration separately at each depth. If $m_d$ depth-$d$
intervals ever become active, their total contribution is bounded by
\[
\min\{\sqrt{T\,m_d\log T},\,m_d\log T/w_d\}.
\]
The first term is the cost of summing square-root deviation terms over the
active intervals at that depth; the second uses the fact that an interval is
refined once its counter reaches the split threshold.
Combining this depth-wise accounting with the split bound from Step 3 yields the
per-segment rates in the main theorem.

\paragraph{Extension to multicalibration.}
For multicalibration, the local constant $c_S$ is replaced by a contextual score
$f_S(x)$ mapping context to label means. The role of ``$r_I$ is above or below $c_S$'' is played by the threshold
comparison $f_S(x_t)\ge r_I$. If these threshold comparisons can be represented as
low-weight linear combinations of the groups, then the groupwise bias control from
Step 1 applies to this more general case as well, where the bias control is worse by a factor of the weight of the linear representation.

\section{A Hierarchy of Complexity Measures}
\label{sec:complexity}


We now define the complexity measures appearing in our bounds, which track various ways in which the sequence of label means can be (un)predictable. In the marginal case, where the context space is a singleton, the predictable mean process reduces to a scalar sequence $q_t:=\mu_t$. The simplest notion of simplicity is piecewise stationarity of the mean --- that it changes only a small number of times:


\begin{definition}[Piecewise-constant predictable means]
\label{def:predmean}
We say the scalar predictable mean sequence changes $J-1$ times if there exist
$
1=\tau_0<\tau_1<\cdots<\tau_J=T+1
$
and values $q^{(1)},\dots,q^{(J)}\in[0,1]$ such that
$q_t=q^{(j)}$ for all $t\in \{\tau_{j-1},\dots,\tau_j-1\}$.
\end{definition}

This benchmark is simple but it is overly conservative. In what follows we generalize it in two ways. First, we should consider a sequence to be simple even if it changes at every round, so long as those changes are only small deviations from some (piecewise) common center. Second, in the multicalibration setting in which there are contexts, there should be ``simple'' sequences that change dramatically at every round so long as those changes are predictable from the contexts in a way that is ``visible'' to the groups with respect to which the algorithm is multicalibrating.

We first introduce the appropriate generalization for multicalibration, which lets the mean change dramatically at every round so long as the changes are predictable from the group functions.  Multicalibration algorithms control bias within every group. Thus what matters is whether
the group family can represent, for each bin midpoint $r_I$, which side of $r_I$ the relevant
conditional means lie on. A sequence where $\mu_t$ changes every round can still be easy if those
changes are predictable from context and their thresholds can be represented using $\bar{\cG}$.

Just as Definition \ref{def:predmean}  required only that the mean be piecewise constant, and allowed it to change with time, we will similarly allow thresholds of label means to be \emph{differently} predictable on different contiguous blocks of time $S$. 
At scale $w_d=2^{-d}$, the useful object on a block $S$ is a score $f$ that approximately predicts
$\mu_t$, takes only $K$ distinct values on the realized contexts, and has thresholds
$\{f(x_t)\ge r\}$ that admit $\bar{\cG}$-representations with $\ell_1$ cost at most $B$. The
product $BK$ controls the number of depth-$d$ splits that can be forced within block $S$:
Lemma~\ref{lm:amortizedsplit} shows that this number is $O(BK)$ when the approximation error is
below scale $w_d$.
A score is threshold representable if each of its level sets has a bounded-cost
linear representation over the group family:

\begin{definition}[Threshold-representable score]
\label{def:threshold-rep}
Fix $S\subseteq[T]$. A score $f:\mathcal X\to[0,1]$ is $(K,B,\eta)$-threshold
representable on $S$ if $|\{f(x_t):t\in S\}|\le K$ and for every $r\in[0,1]$ there exist coefficients
$(\alpha_{r,g})_{g\in\bar{\cG}}$ with
$\sum\limits_{g\in\bar{\cG}}|\alpha_{r,g}|\le B$
such that the function
$h_r(x):=\sum\limits_{g\in\bar{\cG}}\alpha_{r,g}g(x)$
satisfies, for all $t\in S$,
\begin{equation}\label{eq:defthre}
\begin{aligned}
h_r(x_t) &\ge 1-\eta && \text{whenever } f(x_t)\ge r,\\
h_r(x_t) &\le -(1-\eta) && \text{whenever } f(x_t)< r.
\end{aligned}
\end{equation}
\end{definition}
We fix the margin parameter $\eta = 1/4$ throughout to ensure the representation $h_r$ separates by a margin strictly bounded away from zero. Any constant $\eta \in (0, 1)$ would suffice for the subsequent analysis as the specific choice only affects the absolute constants hidden by the Big-O notation.
\begin{remark}\label{rem:threshold-cost-lower-bound}
With $\eta=1/4$, every nonempty block satisfying the threshold-representable condition has
$B\ge 3/4$. Indeed, choose $t_0\in S$ and set $r=f(x_{t_0})$. Then $f(x_{t_0})\ge r$, so
the representation guarantee gives $h_r(x_{t_0})\ge 3/4$. Since every group in $\bar{\cG}$ is
binary-valued,
\[
|h_r(x_{t_0})|\le \sum_g |\alpha_g|\le B.
\]
\end{remark}
The threshold-representable score condition is reminiscent of margin-based complexity notions such as fat-shattering, but it is instance-specific: rather than requiring a hypothesis class to realize arbitrary dichotomies on $S$, we only require low-cost representations of the nested thresholds induced by a chosen low-cardinality score $f$. 
As a sanity check, in the singleton-context case with $\bar{\cG}=\{\mathbf 1\}$, a constant score
$f\equiv c$ is $(1,1,1/4)$-threshold representable.
\begin{example}[Ordinal risk strata]\label{appsec:example}
Suppose that on a block $S$, the predictable mean is well approximated by an ordinal $K$-level
score
\[
f(x)=a_{c(x)},
\qquad
a_1<\cdots<a_K,\quad c(x)\in\{1,\dots,K\}.
\]
Think of $c(x)$ as a coarse risk assessment such as routine, elevated, high, or critical. Assume the
group family contains the all-ones group and the cumulative risk indicators
\[
g_j(x)=\mathbf 1\{c(x)\ge j\},
\qquad
j=2,\dots,K.
\]
Then every threshold of $f$ has constant representation cost. Fix any
$r\in[0,1]$. If $r\le a_1$, then $f(x)\ge r$ for all realized contexts, and we
take $h_r(x)= 1$.
If $r>a_K$, then $f(x)<r$ for all realized contexts, and we take $h_r(x)=-1$.
Otherwise, there is an index $j\in\{2,\dots,K\}$ such that $a_{j-1}<r\le a_j$.
In this case,
\[
f(x)\ge r
\quad\Longleftrightarrow\quad
c(x)\ge j,
\]
so we take
\[
h_r(x)=2g_j(x)- 1.
\]
Thus, on the realized contexts,
\[
\begin{cases}
h_r(x_t)=1, & f(x_t)\ge r,\\
h_r(x_t)=-1, & f(x_t)<r.
\end{cases}
\]
Moreover,
the sum of absolute coefficients is at most $3$, since $h_r$ is either $1$,
$-1$, or $2g_j-1$.
Therefore the ordinal score is $(K,3,0)$-threshold representable on $S$.
In particular, it is also $(K,3,\eta)$-threshold representable for any
$\eta\in[0,1]$.
\end{example}

Next, we introduce the generalization that allows the predictable mean sequence $\mu_t$ to vary from its score $f(x_t)$ (or, in the marginal case, to vary from a constant within a block) so long as those variations are small in aggregate. 
Approximation errors below the resolution
$w_d$ of a depth-$d$ interval should be free, because the algorithm's analysis already pays for error at this scale. This motivates the following residual, which ignores such small errors.

\begin{definition}[Residual at scale $d$ and threshold-simple block]
\label{def:scale-residual}
For a contiguous block $S\subseteq[T]$, a local score $f:\mathcal X\to[0,1]$, and depth $d$, the truncated residual is defined as
\[
\mathsf R_{S,d}(f)
:=
\sum_{t\in S}\bigl(|\mu_t-f(x_t)|-w_d\bigr)_+,
\qquad
(a)_+:=\max\{a,0\}.
\]
Furthermore, we call $S$ $(K,B,d)$-threshold simple if there exists a score
$f_S$ that is $(K,B,\frac14)$-threshold representable on $S$ and satisfies
$\mathsf R_{S,d}(f_S)=0$.
\end{definition}

We now define the complexity measure that our main theorem depends on. For a partition of time into blocks $S$, the measure depends on the best threshold-representable score function $f_S$ that can be defined on each block $S$, and scales with the parameters $K_S$ and $B_S$ of the score function along with its residuals $\mathsf R_{S,d}(f_S)$. It then sums a complexity term depending on these parameters across all blocks $S$ in the partition. Since the algorithm will obtain bounds scaling with the \emph{best} such choice of partition and score functions, the complexity measure takes the infimum across all choices of partition, score functions, and valid values of the complexity parameters. 

\begin{definition}[Depth-wise residual threshold cost]
\label{def:residual-profile}
For each depth $d$, define
\[
\Psi_d^{\rm res}(\mu_{1:T};\cG)
:=
\inf_{\cP_d,\{(f_S,K_S,B_S)\}_{S\in\cP_d}}
\sum_{S\in\cP_d}
\left[
B_SK_S
+
\sqrt{\frac{B_SK_S\,\mathsf R_{S,d}(f_S)\,w_d}{L}}
\right],
\]
where the infimum ranges over partitions $\cP_d$ of $[T]$ into nonempty
contiguous blocks and, for each block $S\in\cP_d$, triples $(f_S,K_S,B_S)$
such that $f_S$ is $(K_S,B_S,\tfrac14)$-threshold representable on $S$ with
respect to $\bar\cG$.
If no admissible representation exists for a block, its cost is $+\infty$.
\end{definition}

If we ask for perfect approximation within scale $d$ (i.e. with zero truncated residual) then the measure simplifies: 

\begin{definition}[Exact multiscale threshold complexity]
\label{def:exact-threshold-complexity}
For a fixed depth $d$, define
\[
M_d^{\rm thr}(\mu_{1:T};\cG)
:=
\inf_{\cP_d,\{(f_S,K_S,B_S)\}_{S\in\cP_d}}
\sum_{S\in\cP_d} B_SK_S,
\]
where the infimum ranges over partitions $\cP_d$ of $[T]$ into nonempty
contiguous blocks and triples $(f_S,K_S,B_S)$ such that $f_S$ is
$(K_S,B_S,\tfrac14)$-threshold representable on $S$ and satisfies
\[
\mathsf R_{S,d}(f_S)=0
\qquad\text{for every }S\in\cP_d.
\]
The exact multiscale threshold complexity is
\[
\Gamma_T^{\rm thr}(\mu_{1:T};\cG)
:=
\sum_{d=0}^{d_{\max}} M_d^{\rm thr}(\mu_{1:T};\cG).
\]
\end{definition}
For a fixed partition and collection of score functions, it is useful to have notation for the relevant complexity terms:
\begin{definition}[Segmented threshold approximation cost]
\label{def:segmented-residual}
For a contiguous partition $\cP$ of $[T]$ and triples
$(f,K,B)=\{(f_S,K_S,B_S):S\in\cP\}$ such that each $f_S$ is
$(K_S,B_S,\tfrac14)$-threshold representable on $S$, define
\[
\mathsf R_S(f_S):=\sum_{t\in S}|\mu_t-f_S(x_t)|,
\]
and
\[
\mathsf M(\cP,f):=1+\sum_{S\in\cP}B_SK_S,
\qquad
\mathsf A(\cP,f):=\sum_{S\in\cP}\sqrt{B_SK_S\,\mathsf R_S(f_S)}.
\]
\end{definition}

\section{Main Theorem and Marginal Corollaries}
\label{sec:results}

We now state the main guarantee. The theorem has two forms. The first is a depth-by-depth bound in
terms of the depth-$d$ costs $\Psi_d^{\rm res}$, which is our tightest bound and closest to the proof. The second
is a cleaner consequence stated in terms of one contiguous partition and one threshold-representable
score per segment. Table~\ref{tab:regimes} summarizes representative consequences.

\begin{table}[t]
\centering
\footnotesize
\setlength{\tabcolsep}{4pt}
\renewcommand{\arraystretch}{1.25}
\begin{tabular}{@{}
>{\raggedright\arraybackslash}p{0.23\textwidth}
>{\raggedright\arraybackslash}p{0.43\textwidth}
>{\raggedright\arraybackslash}p{0.28\textwidth}
@{}}
\toprule
\textbf{Setting} & \textbf{Complexity control} & \textbf{Calibration error} \\
\midrule
Worst-case adversarial &
No structure assumed&
$\widetilde O(T^{2/3})$ \\
\addlinespace[2pt]
Exact threshold score \newline
{\scriptsize ($K$ levels, cost $B$)} &
\(\begin{aligned}[t]
\mathsf M(\cP,f)&\lesssim 1+BK,\\
\mathsf A(\cP,f)&=0
\end{aligned}\) &
$\widetilde O(\sqrt{BKT})$ \\
\addlinespace[2pt]
Approximate threshold score \newline
{\scriptsize ($K$ levels, cost $B$)} &
\(\begin{aligned}[t]
\mathsf M(\cP,f)&\lesssim 1+BK,\\
\mathsf A(\cP,f)&\lesssim\sqrt{BKR},\\
R&=\sum_t|\mu_t-f(x_t)|
\end{aligned}\) &
\(\widetilde O(\sqrt{BKT}+(TBKR)^{1/3})\) \\
\addlinespace[2pt]
Segmentwise approximate threshold score&
\(\begin{aligned}[t]
\mathsf M(\cP,f)&=1+\sum_S B_SK_S,\\
\mathsf A(\cP,f)&=\sum_S\sqrt{B_SK_S\mathsf R_S(f_S)}
\end{aligned}\) &
\(\begin{aligned}[t]
\widetilde O(&\sqrt{T\mathsf M(\cP,f)}\\
&{}+T^{1/3}\mathsf A(\cP,f)^{2/3})
\end{aligned}\) \\
\addlinespace[2pt]
$J$-piecewise stationary marginal &
\(\begin{aligned}[t]
\mathsf M(\cP,f)&=1+J,\\
\mathsf A(\cP,f)&=0
\end{aligned}\) &
$\widetilde O(\sqrt{JT})$ \\
\addlinespace[2pt]
Approximately stationary marginal &
\(\begin{aligned}[t]
\mathsf M(\cP,f)&=O(1),\\
\mathsf A(\cP,f)&=\sqrt{C_{\rm stat}(q)}
\end{aligned}\) &
$\widetilde O(\sqrt T+(TC_{\rm stat}(q))^{1/3})$ \\
\bottomrule
\end{tabular}
\caption{Representative regimes covered by the dynamic-bin algorithm. The worst-case row recovers
the known $\widetilde O(T^{2/3})$ online multicalibration rate achieved by prior
algorithms~\citep{noarov2023high,ghuge2025improved}. In the last row,
$C_{\rm stat}(q):=\inf_{c\in[0,1]}\sum_t |q_t-c|$. The algorithm does not need to know which regime it is running in; all of these bounds are obtained by the same algorithm.}
\label{tab:regimes}
\end{table}

\begin{theorem}[Instance-adaptive online multicalibration]
\label{thm:main-mc}
For $d=0,\dots,d_{\max}$, define
\[
\Theta_0:=1,
\qquad
\Theta_d
:=
\min\left\{
2^d,\ \frac{T w_d^2}{L},\ \Psi_{d-1}^{\rm res}(\mu_{1:T};\cG)
\right\}
\quad(d\ge 1).
\]
With probability at least $1-O(1/T)$, the dynamic-bin algorithm satisfies
\[
\MCerr_{\cG}(T)
\le
C
\min\left\{
T^{2/3}L^{1/3},
\sum_{d=0}^{d_{\max}}
\min\left\{
\sqrt{TL\Theta_d},\frac{L\Theta_d}{w_d}
\right\}
\right\}.
\]
Moreover, on the same high-probability event, for every contiguous partition $\cP$
of $[T]$ and every choice of triples
$(f,K,B)=\{(f_S,K_S,B_S):S\in\cP\}$ such that each $f_S$ is
$(K_S,B_S,\tfrac14)$-threshold representable on $S$,
\[
\MCerr_{\cG}(T)
\le
C\min\left\{
T^{2/3}L^{1/3},
\sqrt{TL\,\mathsf M(\cP,f)}
+
T^{1/3}L^{1/3}\mathsf A(\cP,f)^{2/3}
\right\}.
\]
\end{theorem}

Since $\MCerr_{\cG}(T)\le T$ deterministically, the same bounds also hold in
expectation up to an additive $O(1)$ term.

The main theorem gives a versatile bound, but can be confusing because of its generality. We now state a couple of more easily interpretable corollaries.  For example, if the entire predictable mean process is exactly described by a single score function with low multiscale threshold complexity then we get the following corollary:

\begin{corollary}[Exact threshold-simple multicalibration]
\label{cor:thresholdsimple}
If the predictable mean process has exact multiscale threshold complexity
$\Gamma_T^{\rm thr}(\mu_{1:T};\cG)$ as in
Definition~\ref{def:exact-threshold-complexity}, then, with probability at least
$1-O(1/T)$,
\[
\MCerr_{\cG}(T)
\le
\widetilde O\left(
\min\left\{
T^{2/3},
\sqrt{T\left(1+\Gamma_T^{\rm thr}(\mu_{1:T};\cG)\right)}
\right\}
\right).
\]
In particular, if for every depth $d$ the whole horizon admits a
$(K_d,B_d,d)$-threshold-simple score with $B_d\le B$ and $K_d\le K$, then
\[
\MCerr_{\cG}(T)
\le
\widetilde O\left(\min\left\{T^{2/3},\sqrt{BKT}\right\}\right).
\]
\end{corollary}

Similarly, suppose there is no group structure (so we are in the marginal calibration special case) and the sequence of label means is piecewise constant on $J$ pieces. Then we get the following: 

\begin{corollary}[Piecewise-stationary marginal calibration]
\label{cor:marginal-l1}
In the context-free marginal calibration setting ($\mathcal{X}$ is a singleton and
$\mathcal{G}=\{1\}$), suppose the predictable mean $q_t:=\mu_t$ is constant on
$J$ contiguous segments. Then, with probability at least $1-O(1/T)$, the dynamic-bin algorithm satisfies
\[
\Cal(T)\le \widetilde O\!\left(\min\{T^{2/3},\sqrt{JT}\}\right).
\]
\end{corollary}

Similarly, suppose the label mean may change at every round, but on average stays close to its median. Then we get the following corollary. 

\begin{corollary}[Slowly drifting marginal calibration]
\label{cor:marginal-special}
In the context-free marginal calibration setting ($\mathcal{X}$ is a singleton and
$\mathcal{G}=\{1\}$), let
$C_{\rm stat}(q)=\inf_{c\in[0,1]}\sum_{t=1}^T |q_t-c|$ be the best $L_1$ deviation from a
constant mean. Then, with probability at least $1-O(1/T)$,
\[
\Cal(T)
\le
\widetilde O\!\left(
\sqrt T+(T C_{\rm stat}(q))^{1/3}
\right).
\]
\end{corollary}

We note that it is possible to mix and match these bounds --- for example, we can read off from the main theorem a bound for the marginal calibration case in which there are $J$ blocks on each of which the label mean stays close to its median, but for which the median is very different on each of the $J$ blocks. This emphasizes that the algorithm is the same for each of these corollaries --- they are just different specializations of the main theorem --- and so we always get the best of any of the stated bounds. 

\section{Proof of the Main Theorem}
\label{sec:proofs}
\paragraph{Proof Overview.}
The proof has six steps. First, we establish that the sleeping-experts wrapper controls one-sided weighted bias for
every group, interval, and contiguous sub-block. Second, we show that if a depth-$d$ interval midpoint is far
from all values of a threshold-representable score on a block, then playing that interval grows 
signed bias at a linear rate; when $\mu_t$ is not exactly the score, the extra cost is the residual approximation
error on the rounds when that interval is played. Together these facts limit how frequently an interval far from a score value can be played. Third, we use this play bound to control how many
depth-$d$ intervals can be close to splitting during a block $S$. Formally, we bound
\[
\Xi_d(S):=\sum_{I\in\cD_d}\min\left\{1,\frac{N_S(I)w_d^2}{L}\right\},
\]
where an interval contributes $1$ if it receives enough mass on $S$ to reach the depth-$d$ split
threshold, and contributes proportionally otherwise. Grouping intervals by distance from the score
values gives
\[
\Xi_d(S)\lesssim BK+\sqrt{BK\,\mathsf R_{S,d}(f)w_d/L}.
\]
Fourth, we show these blockwise bounds control the number of intervals that ever become active at each depth.
Fifth, calibration is accounted for by separately bounding the contribution of predictions made using intervals at each depth.  Finally, a dyadic
summation lemma converts the depth-wise bounds into the clean segmented approximation rate in
Theorem~\ref{thm:main-mc}.

To separate the probabilistic terms from the later split-count analysis, we define
a global concentration event
$\mathcal{E}_{\text{global}} := \mathcal{E}_{\text{samp}} \cap \mathcal{E}_{\text{conc}}$,
and we further define $N_S^g(I):=\sum_{t\in S} g(x_t)\pi_{t,I}$. These events
separate the stochasticity into two sources:
\begin{itemize}
    \item \textbf{Sampling Error ($\mathcal{E}_{\text{samp}}$):} This event accounts for the variance introduced by the learner's randomized point predictions $p_t \sim \pi_t$. It ensures that for any group $g$ and interval $I$, the realized weighted outcomes stay close to their expected values under the mixed forecast:
    \begin{equation}
        \mathcal{E}_{\text{samp}} : \max_{g \in \bar{\mathcal{G}}, I \in \mathcal{D}} \left| \sum_{t=1}^T (\mathbf{1}[p_t = r_I] - \pi_{t,I}) g(x_t)(y_t - r_I) \right| \le C_{\text{samp}} \left( \sqrt{N_T^g(I) L} + L \right),
    \end{equation}
    where $N_T^g(I) = N_{[T]}^g(I)$.
    \item \textbf{Martingale Concentration ($\mathcal{E}_{\text{conc}}$):} This event bounds the outcome noise relative to the predictable means $\mu_t$. It ensures that on any contiguous block $S$, the realized outcomes do not drift too far from their conditional means:
    \begin{equation}
    \mathcal{E}_{\text{conc}} : \max_{\substack{g \in \bar{\mathcal{G}},\, I \in \mathcal{D}\\ S \subseteq [T]\ \mathrm{contiguous}}} \left| \sum_{t \in S} g(x_t) \pi_{t,I} (y_t - \mu_t) \right| \le C_{\text{conc}} \left( \sqrt{N_S^g(I) L} + L \right)
    \end{equation}
\end{itemize}
By applying Freedman's inequality (see Appendix~\ref{app:concentration}), each event fails with probability at most $1/4T$, ensuring $\Pr(\mathcal{E}_{\text{global}}) \ge 1 - \frac{1}{2T}$. Our analysis will condition on this event.


\subsection{From Bias Control via Regret to Depth-Wise Calibration}
To bound the global multicalibration error, we first 
control cumulative bias
within each activated interval. This follows from the regret-minimization properties
of our wrapper algorithm.
\begin{lemma}[Bias control via regret]\label{lem:bias-control-via-regret}
For every group $g\in\bar{\cG}$, every dyadic interval $I$, and every contiguous block
$S\subseteq A(I)$,
\begin{equation}
\left|\sum_{t\in S} g(x_t)\pi_{t,I}(y_t-r_I)\right|\le N_S^g(I)w_I + C_{\rm loc}\left(\sqrt{N_S^g(I)L}+L\right).
\end{equation}
\end{lemma}
\begin{proof}[Proof Sketch]
We bound the cumulative one-sided weighted bias by relating it to the wrapper's
regret against the sleeping expert associated with $\phi^+_{g,I}$ and start time
$\min S$. Since the
feasibility condition (\ref{eq:weighted-feas}) ensures the aggregate one-step bias
satisfies $\widehat \phi_t \le 0$, regret against this expert upper-bounds the
cumulative positive bias term $\sum_{t\in S}\phi^+_{g,I}(\pi_t,y_t)$.
The first-order AdaNormalHedge guarantee scales with the expert's positive
relative loss, and in this construction that quantity is at most a constant
multiple of the total group-weighted play $N_S^g(I)$. Therefore the positive
one-sided bias is at most $O(\sqrt{N_S^g(I)L}+L)$. Adding back the
width slack in the definition of $\phi^+_{g,I}$ gives
\[
\sum_{t\in S}g(x_t)\pi_{t,I}(y_t-r_I)
\le
N_S^g(I)w_I+O(\sqrt{N_S^g(I)L}+L).
\]
The negative direction is identical using the expert associated with
$\phi^-_{g,I}$. See Appendix~\ref{app:biascontrolviaregret} for details.
\end{proof}

\begin{lemma}[Depth-wise calibration accounting]
\label{lem:depth-calibration}
Let $\cV_d:=\cV_T\cap\cD_d$ be the set of depth-$d$ intervals that are active in at least one
prediction round, and let $m_d:=|\cV_d|$. Assume $L\le T$. On $\cE_{\rm samp}$, for every group
$g\in\cG$ and every depth $d$,
\[
\sum_{I\in\cV_d}
\left|
\sum_{t:p_t=r_I}g(x_t)(y_t-r_I)
\right|
\le
C\min\left\{\sqrt{TLm_d},\frac{Lm_d}{w_d}\right\}.
\]
\end{lemma}
\begin{proof}[Proof Sketch]
Apply Lemma~\ref{lem:bias-control-via-regret} to each interval $I\in\cV_d$ on its active lifetime
and combine it with the sampling event. The split rule implies
$N(I)w_d^2\le CL$ for every interval that ever appears: a non-max-depth interval
can overshoot the split threshold by at most one round of mass, while a max-depth
interval has $N(I)\le T$ and $T w_d^2=O(1)$. Hence the
deterministic discretization term is absorbed by the square-root term. Summing
$\sqrt{N_T^g(I)L}$ over $I\in\cV_d$ gives the first bound by Cauchy--Schwarz and the second bound
using $N_T^g(I)\le N(I)\le CL/w_d^2$. See Appendix~\ref{app:depthcalibration}.
\end{proof}


\subsection{Controlling Splits From Approximation Structure}

We now connect the approximation costs to the adaptive tree. The key point is that if an
interval midpoint is far from all score values on a block, then the interval can be played often
only when a significant amount of residual approximation error is assigned to it.

\begin{lemma}[Intervals far from the score are rarely played]
\label{lem:high-bias-rarely-played}
Condition on $\cE_{\text{conc}}$. Fix a depth $d$, write $w=w_d$, let $S$ be a contiguous block,
and let $f$ be $(K,B,\tfrac14)$-threshold representable on $S$. Let
$C_f(S):=\{f(x_t):t\in S\}$. For a depth-$d$ interval $I$ with midpoint $r_I$, define
\[
\delta_I:=\operatorname{dist}(r_I,C_f(S)),
\qquad
\mathsf D_{S,I}:=
\sum_{t\in S\cap A(I)}
\pi_{t,I}\bigl(|\mu_t-f(x_t)|-w\bigr)_+.
\]
There is a universal constant $C_0$ such that, whenever $\delta_I>C_0Bw$,
\[
N_S(I)
\le
C\left[
\frac{B^2L}{(\delta_I-C_0Bw)^2}
+
\frac{B\mathsf D_{S,I}}{\delta_I-C_0Bw}
\right].
\]
The trivial bound $N_S(I)\le |S|$ always holds.
\end{lemma}
\begin{proof}[Proof Sketch]
Let $h_{r_I}=\sum_g\alpha_g g$ be the threshold representation of $f$ at threshold $r_I$. The
margin condition gives
$h_{r_I}(x_t)(f(x_t)-r_I)\ge (3/4)|f(x_t)-r_I|$, while
$|h_{r_I}(x_t)|\le B$. Therefore
\[
h_{r_I}(x_t)(\mu_t-r_I)
\ge
(3/4)\delta_I-Bw-B\bigl(|\mu_t-f(x_t)|-w\bigr)_+.
\]
After summing with weights $\pi_{t,I}$ over $S\cap A(I)$, this lower bounds the expected signed bias by a linear term in $N_S(I)$ minus $B\mathsf D_{S,I}$. The algorithmic upper bound follows by expanding $h_{r_I}$ in the group basis and applying Lemma~\ref{lem:bias-control-via-regret} and $\cE_{\rm conc}$ to each group. Combining the two bounds and solving the resulting quadratic gives the claim. See Appendix~\ref{app:highbiasrarelyplayed}.
\end{proof}

We next convert control of the total play on a single block into a bound on how much that block can
contribute toward splits at depth $d$.
\begin{lemma}[Split count on one block]\label{lm:amortizedsplit}
Condition on $\cE_{\text{conc}}$. Fix depth $d$, write $w=w_d$, and let $S$ be a contiguous
block. If $f$ is $(K,B,\tfrac14)$-threshold representable on $S$, then
\[
\Xi_d(S)
:=
\sum_{I\in\cD_d}
\min\left\{1,\frac{N_S(I)w_d^2}{L}\right\}
\le
C\left(
BK+\sqrt{\frac{BK\,\mathsf R_{S,d}(f)\,w}{L}}
\right).
\]
\end{lemma}
\begin{proof}[Proof Sketch]
We group depth-$d$ intervals according to their distance from the nearest realized
score value. Intervals within distance $O(Bw_d)$ are few, contributing at
most $O(BK)$ to $\Xi_d(S)$. For farther intervals, Lemma~\ref{lem:high-bias-rarely-played}
gives a bound with two terms: a threshold-cost term that decays quadratically in
the distance, and a residual term depending on the residual mass assigned to that
interval. Summing over distance ranges and optimizing the cutoff between nearby
and far intervals gives the claimed bound. See Appendix~\ref{app:splitcount}.
\end{proof}

Let $\cA_d$ denote the set of depth-$d$ intervals that split by time $T$.

\begin{lemma}[Active intervals from split bounds]\label{lem:active-profile}
Condition on $\cE_{\rm conc}$ and let $m_d=|\cV_d|$. Then
\[
m_0=1,
\qquad
m_d
\le
C\min\left\{
2^d,\frac{T w_d^2}{L},\Psi_{d-1}^{\rm res}(\mu_{1:T};\cG)
\right\}
\quad(d\ge1).
\]
\end{lemma}
\begin{proof}[Proof Sketch]
Every active depth-$d$ interval is generated by a depth-$(d-1)$ parent that split,
so $m_d\le 2|\mathcal A_{d-1}|$. The deterministic caps $m_d\le 2^d$ and
$m_d\le C T w_d^2/L$ follow from the dyadic grid and the total play mass. For the
instance-dependent bound, split mass is additive over any contiguous temporal
partition, and Lemma~\ref{lm:amortizedsplit} bounds the resulting blockwise split
costs. Taking the infimum over partitions and score representations gives the
bound in terms of $\Psi_{d-1}^{\rm res}$. See Appendix~\ref{app:activeprofile}.
\end{proof}

\begin{lemma}[Dyadic summation]
\label{lem:dyadic-summation}
Let $w_d=2^{-d}$, and let $M\ge1$ and $A\ge0$. Suppose $m_0\le C$ and, for every $d\ge1$,
\[
m_d
\le
C\min\left\{
2^d,\frac{T w_d^2}{L},M+A\sqrt{\frac{w_d}{L}}
\right\}.
\]
Then
\[
\sum_{d=0}^{d_{\max}}
\min\left\{
\sqrt{TLm_d},\frac{Lm_d}{w_d}
\right\}
\le
C\left(
\sqrt{TLM}
+
T^{1/3}L^{1/3}A^{2/3}
\right).
\]
\end{lemma}
\begin{proof}[Proof Sketch]
The deterministic caps give the worst-case contribution
$O(T^{2/3}L^{1/3})$. For the instance-dependent contribution, split
$M+A\sqrt{w_d/L}$ into an $M$ part and an $A$ part. The $M$ part sums to
$O(\sqrt{TLM})$ by a dyadic crossover between $Tw_d$ and $LM/w_d$. The $A$ part
sums to $O(T^{1/3}L^{1/3}A^{2/3})$ by the crossover
$w_*=(A\sqrt L/T)^{2/3}$. See Appendix~\ref{app:dyadicsum} for the details.
\end{proof}

\subsection{Proof of the Main Theorem}

\begin{proof}[Proof of Theorem~\ref{thm:main-mc}]
If $L>T$, then the trivial bound $\MCerr_{\cG}(T)\le T$ is covered by the
worst-case term $T^{2/3}L^{1/3}$; it is also covered by the depth-wise branch
because the depth-$0$ term is at least $T$. Hence assume $L\le T$ and condition on
$\cE_{\rm global}$.

Each realized prediction value is the midpoint of a unique interval that is active
at some time. Group these ever-active intervals by their depth. For any fixed
group $g\in\cG$, 
\[
\sum_{v}
\left|\sum_{t:p_t=v}g(x_t)(y_t-v)\right|
=
\sum_{d=0}^{d_{\max}}\sum_{I\in\cV_d}
\left|\sum_{t:p_t=r_I}g(x_t)(y_t-r_I)\right|.
\]
Applying Lemma~\ref{lem:depth-calibration} at each depth and then maximizing over
$g$ gives
\[
\MCerr_{\cG}(T)
\le
C\sum_{d=0}^{d_{\max}}
\min\left\{\sqrt{TLm_d},\frac{Lm_d}{w_d}\right\}.
\]
Lemma~\ref{lem:active-profile} gives $m_d\le C\Theta_d$ for every depth $d$.
Absorbing constants into $C$ proves the depth-wise bound in the theorem. The
deterministic parts of $\Theta_d$ also give the worst-case branch. Indeed, for
$d\ge1$ we use only
\[
\Theta_d\le 2^d=\frac1{w_d},
\qquad
\Theta_d\le \frac{T w_d^2}{L}.
\]
Thus the $d$th summand is at most
\[
\min\left\{
\sqrt{\frac{TL}{w_d}},
T w_d
\right\}.
\]
Let $w_*:=(L/T)^{1/3}$, where the two terms are equal. For depths with
$w_d\ge w_*$, the terms $\sqrt{TL/w_d}$ form a geometrically increasing sequence
as $d$ increases, so their total contribution is at most a constant times their
value at $w_*$. For depths with $w_d<w_*$, the terms $T w_d$ form a geometrically
decreasing tail, so their total contribution is at most a constant times
$T w_*$. Since
\[
\sqrt{\frac{TL}{w_*}}=T w_*=T^{2/3}L^{1/3},
\]
and the depth-$0$ term is $O(\sqrt{TL})\le O(T^{2/3}L^{1/3})$ when $L\le T$, we
obtain
\[
\sum_{d=0}^{d_{\max}}
\min\left\{\sqrt{TL\Theta_d},\frac{L\Theta_d}{w_d}\right\}
\le
C T^{2/3}L^{1/3}
\]
as claimed.

For the segmented consequence, fix an admissible partition $\cP$ and triples
$(f,K,B)=\{(f_S,K_S,B_S):S\in\cP\}$. For every $d\ge1$, since
$\mathsf R_{S,d-1}(f_S)\le \mathsf R_S(f_S)$,
\[
\Psi_{d-1}^{\rm res}(\mu_{1:T};\cG)
\le
\sum_{S\in\cP}
\left[
B_SK_S+\sqrt{\frac{B_SK_S\,\mathsf R_S(f_S)\,w_{d-1}}{L}}
\right]
\le
\mathsf M(\cP,f)+\sqrt2\,\mathsf A(\cP,f)\sqrt{\frac{w_d}{L}},
\]
where the last step uses $w_{d-1}=2w_d$. 
Combining this with
Lemma~\ref{lem:active-profile},
the active interval counts satisfy the hypothesis
of Lemma~\ref{lem:dyadic-summation} with $M=\mathsf M(\cP,f)$ and $A$ equal to
an absolute-constant multiple of $\mathsf A(\cP,f)$. Applying that lemma gives
the segmented bound. The worst-case cap is the first branch proved above.
\end{proof}

\begin{proof}[Proof of Corollary~\ref{cor:thresholdsimple}]
Use the depth-wise part of Theorem~\ref{thm:main-mc}. For each depth $d$, every
exact threshold-simple partition witnessing $M_d^{\rm thr}$ has zero
scale-truncated residual, hence
$\Psi_d^{\rm res}(\mu_{1:T};\cG)\le M_d^{\rm thr}(\mu_{1:T};\cG)$.
Therefore, for $d\ge1$, $\Theta_d\le M_{d-1}^{\rm thr}$, while $\Theta_0=1$.
The depth-wise sum is bounded by
\[
\sqrt{TL}
+
\sum_{d=1}^{d_{\max}}\sqrt{TL\,M_{d-1}^{\rm thr}}
\le
\widetilde O\!\left(
\sqrt{TL\left(1+\Gamma_T^{\rm thr}(\mu_{1:T};\cG)\right)}
\right),
\]
where the last step uses Cauchy--Schwarz over the $O(\log T)$ depths. Taking the
minimum with the worst-case branch proves the first claim.

For the special case, the one-block partition at each depth gives
$M_d^{\rm thr}\le B_dK_d\le BK$. Hence
$\Gamma_T^{\rm thr}\le O(BK\log T)$, and the logarithmic factor is hidden in
$\widetilde O(\cdot)$.
\end{proof}

\begin{proof}[Proof of Corollary~\ref{cor:marginal-l1}]
Let the stationary segments be $S_1,\dots,S_J$, and write $q_t=q^{(j)}$ on
$S_j$. Let $\cP=\{S_1,\dots,S_J\}$. On each segment $S_j$, use the constant score
$f_{S_j}\equiv q^{(j)}$. In the
singleton-context setting, each such score is $(1,1,\tfrac14)$-threshold
representable and has zero residual. Thus $\mathsf M(\cP,f)=1+J$ and
$\mathsf A(\cP,f)=0$. The segmented part of Theorem~\ref{thm:main-mc} gives
\[
\Cal(T)
\le
\widetilde O\left(\min\left\{T^{2/3},\sqrt{JT}\right\}\right).
\]
\end{proof}

\begin{proof}[Proof of Corollary~\ref{cor:marginal-special}]
Let $c^*\in\arg\min_{c\in[0,1]}\sum_{t=1}^T|q_t-c|$. Use the one-block partition
$\cP=\{[T]\}$ and the constant score $f\equiv c^*$. In the singleton-context
setting this score is $(1,1,\tfrac14)$-threshold representable. Therefore
$\mathsf M(\cP,f)=O(1)$ and
$\mathsf A(\cP,f)=\sqrt{C_{\rm stat}(q)}$. The segmented part of
Theorem~\ref{thm:main-mc} gives
\[
\Cal(T)
\le
\widetilde O\left(
\sqrt T+T^{1/3}C_{\rm stat}(q)^{1/3}
\right)
=
\widetilde O\left(
\sqrt T+(T C_{\rm stat}(q))^{1/3}
\right).
\]
\end{proof}

\section{Tightness of the Threshold-Complexity Dependence}
\label{sec:tightness}

We now show that the exact threshold-complexity corollary is tight, up to polylogarithmic factors,
along the whole curve
$\min\{\sqrt{T\Gamma_T^{\rm thr}},T^{2/3}\}$. The construction is a parameterized version
of the Walsh lower-bound instance of \citet[Section~4]{collina2026optimal}. In their instance, a  grid discretization parameter 
$m$ controls the threshold complexity: on an ordered grid of size $m$, every group family
has $\Gamma_T^{\rm thr}\ge \widetilde\Omega(m)$, while the Walsh family has
$\Gamma_T^{\rm thr}=\widetilde\Theta(m)$. The lower-bound proof of
\citet{collina2026optimal} then gives multicalibration error
$\widetilde\Omega(\sqrt{mT})$ for every $m\le T^{1/3}$.

\subsection{The Ordered-Grid Problem Instance.}
To establish the lower bound, we construct an instance where the contexts are restricted to a
uniformly spaced one-dimensional grid. Fix a power of two $m$ satisfying
$2\le m\le T^{1/3}$. 
The contexts $x_1, \dots, x_m$ are uniformly spaced in the interval $[1/4, 3/4]$:
\begin{equation}
x_i := \frac{1}{4} + \frac{i-1}{2(m-1)}, \quad i \in \{1, \dots, m\}.
\end{equation}
The learner faces these contexts in a periodic, round-robin fashion, i.e., $x_t := x_{1 + ((t-1) \pmod m)}$. The outcomes $y_t$ are generated as:
\begin{equation}
y_t = x_t + \frac{\xi_t}{4}, \quad \xi_t \in \{\pm 1\} \text{ (independent Rademacher noise)}.
\end{equation}
Under this construction, the predictable mean is simply the context itself, $\mu_t = x_t$, and the optimal regression function in the class is the identity $f^\star(x) = x$.

\subsection{Fundamental Limits of Threshold Complexity}
\begin{lemma}[Generic fine-scale threshold lower bound]
\label{lem:generic-threshold-lb}
For the ordered-grid instance above and any group family $\cG$,
\[
\Gamma_T^{\rm thr}(\mu_{1:T};\cG)\ge c'\,m\log T =\widetilde\Omega(m),
\]
for some constant $c'>0$.
\end{lemma}

\begin{lemma}[Threshold complexity of the Walsh family]\label{lem:walsh-complexity}
There exists a binary group family $\mathcal{G}_{T,m}$ of size $O(\log^3 m)$ such that for
the ordered-grid instance, the threshold complexity satisfies:
\begin{equation}
\Gamma_T^{\rm thr}(\mu_{1:T}; \mathcal{G}_{T,m}) = \widetilde{\Theta}(m).   
\end{equation}
\end{lemma}

\subsection{Tightness of Multicalibration Error}
The next theorem is a parameterized restatement of the
Walsh lower bound of \citet{collina2026optimal}. We do not reprove it from scratch; we only track the dependencies in its proof on the grid size $m$, whereas \citet{collina2026optimal} state their theorem with $m$ optimized for the worst-case bound.

\begin{theorem}[Parameterized Walsh-family lower bound]
\label{thm:collina-imported}
There exist universal constants $c,C>0$ such that, for all sufficiently large $T$ and every
power of two $m$ with $2\le m\le T^{1/3}$, every (possibly randomized) online forecaster on
the ordered-grid instance above and the Walsh family $\cG_{T,m}$ satisfies
\[
\mathbb E[\MCerr_{\cG_{T,m}}(T)]\ge c\,\frac{\sqrt{mT}}{\log^C(T+1)}.
\]
Moreover, by Lemma~\ref{lem:walsh-complexity}, this same instance has
\[
\Gamma_T^{\rm thr}(\mu_{1:T};\cG_{T,m})=\widetilde\Theta(m),
\]
so the lower bound can equivalently be written as
\[
\mathbb E[\MCerr_{\cG_{T,m}}(T)]
\ge
\widetilde\Omega\!\left(
\sqrt{T\,\Gamma_T^{\rm thr}(\mu_{1:T};\cG_{T,m})}
\right).
\]
Consequently, up to powers of two and logarithmic factors, the bound in
Corollary~\ref{cor:thresholdsimple} is tight throughout the regime
$\Gamma_T^{\rm thr}\le T^{1/3}$. Taking $m\asymp T^{1/3}$ gives the worst-case lower
bound $\widetilde\Omega(T^{2/3})$, which is the cap in Theorem~\ref{thm:main-mc}.
Equivalently, for every threshold-complexity budget $\Gamma_\star\ge 2$, there are
instances with $\Gamma_T^{\rm thr}\le \widetilde O(\Gamma_\star)$ and expected error at least
\[
\widetilde\Omega\!\left(\min\{\sqrt{T\Gamma_\star},T^{2/3}\}\right).
\]
\end{theorem}

Proofs of Lemmas~\ref{lem:generic-threshold-lb},~\ref{lem:walsh-complexity}, and
Theorem~\ref{thm:collina-imported} appear in
Appendix~\ref{appsec:tightness-proofs}. Conversely, plugging
Lemma~\ref{lem:walsh-complexity} into
Corollary~\ref{cor:thresholdsimple}, and converting the high-probability
bound to expectation using the failure-event argument from Theorem~\ref{thm:main-mc}, gives
the matching upper bound
\[
\mathbb E[\MCerr_{\cG_{T,m}}(T)]
\le
\widetilde O\!\left(\sqrt{mT}\right).
\]
Therefore the upper and lower bounds match, up to polylogarithmic factors, along the full
curve obtained by varying $m$ up to the critical scale $T^{1/3}$, and the largest such $m$
matches the worst-case bound.
This is consistent with the residual bound in Theorem~\ref{thm:main-mc}: the Walsh instances have exact low-residual
scores, and the hardness lies in the threshold cost of representing their thresholds rather than in
residual drift.

\subsection*{Acknowledgments}
We gratefully acknowledge support from the Simons Collaboration on Algorithmic Fairness and the NSF EnCoRE Tripods institute.

\bibliographystyle{plainnat}
\bibliography{reference}

\appendix

\section{Concentration Inequality}
\subsection{Two-stage Freedman Inequality (Lemma~\ref{lem:freedman})
}
\begin{lemma}[Two-stage Freedman inequality]
\label{lem:freedman}
Let $(\cH_{t-1},\cF_t)_{t=1}^n$ be a two-stage filtration with
$\cH_{t-1}\subseteq \cF_t\subseteq \cH_t$ for every $t$, and let $\Delta_t$ be $\cF_t$-measurable random
variables satisfying
\[
\mathbb E[\Delta_t\mid \cH_{t-1}]=0,
\qquad
|\Delta_t|\le 1
\quad\text{a.s.}
\]
for every $t$. Then for every $x,v>0$,
\[
\Pr\!\left(
\left|\sum_{t=1}^n \Delta_t\right|\ge x
\text{ and }
\sum_{t=1}^n \mathbb E[\Delta_t^2\mid \cH_{t-1}] \le v
\right)
\le
2\exp\!\left(-\frac{x^2}{2(v+x/3)}\right).
\]
\end{lemma}

\begin{proof}
We reduce the two-stage case to the standard Freedman inequality by constructing an interleaved filtration. Set $\widetilde{\cF}_0:=\cF_0$, and define $(\widetilde{\cF}_k)_{k=1}^{2n}$ as follows:
\begin{equation}
\widetilde{\cF}_{2t-1}:=\cH_{t-1},
\qquad
\widetilde{\cF}_{2t}:=\cF_t,
\qquad
t=1,\dots,n.    
\end{equation}
Correspondingly, define the sequence of increments $(D_k)_{k=1}^{2n}$ by $D_{2t-1} := 0$ and $D_{2t} := \Delta_t$. Let $M_k := \sum_{s=1}^k D_s$ be the associated partial sums. By construction, $D_k$ is $\widetilde{\cF}_k$-measurable and satisfies:
\begin{equation}
\mathbb{E}[D_k \mid \widetilde{\cF}_{k-1}] = 
\begin{cases} 
\mathbb{E}[0 \mid \cF_{t-1}] = 0 & \text{if } k=2t-1, \\
\mathbb{E}[\Delta_t \mid \cH_{t-1}] = 0 & \text{if } k=2t.
\end{cases}
\end{equation}
Then $(M_k,\widetilde{\cF}_k)$ is a martingale, its increments are bounded by $1$ in absolute value,
and its predictable quadratic variation at time $2n$ is
\begin{equation}
\langle M \rangle_{2n} = \sum_{k=1}^{2n} \mathbb{E}[D_k^2 \mid \widetilde{\cF}_{k-1}] = \sum_{t=1}^n \mathbb{E}[\Delta_t^2 \mid \cH_{t-1}].
\end{equation}
Apply the standard one-sided Freedman inequality to $M_k$ and to $-M_k$, then union bound the two
tails; see \citet[Theorem~1.1]{10.1214/ECP.v16-1624}.
\end{proof}
\subsection{The Global Concentration Event}\label{app:concentration}
\subsubsection{Sampling Error 
(Lemma~\ref{lem:group-sampling})
}
\begin{lemma}[Sampling error]
\label{lem:group-sampling}
There exists an event $\cE_{\rm samp}$ of probability at least $1-\frac{1}{4T}$ such that, on
$\cE_{\rm samp}$, for every $g\in\bar{\cG}$ and every interval $I\in\cD$,
\begin{equation}
\left|\sum_{t=1}^T\bigl(\mathbf 1[p_t=r_I]-\pi_{t,I}\bigr)g(x_t)(y_t-r_I)\right|
\le C_{\rm samp}\bigl(\sqrt{N_T^g(I)L}+L\bigr).    
\end{equation}
\end{lemma}

\begin{proof}
Fix a group $g \in \bar{\cG}$ and an interval $I \in \cD$. To bound the sampling error, we analyze the concentration of the sequence
\begin{equation}
    Z_t^g(I) := \bigl(\mathbf 1[p_t=r_I]-\pi_{t,I}\bigr)g(x_t)(y_t-r_I).
\end{equation}
The proof is based on an application of Freedman's inequality with the peeling technique to handle the random conditional variance.

\textbf{Step 1: Martingale properties and variance bounds.}
Let $M_T^g(I) := \sum_{t=1}^T Z_t^g(I)$. By the protocol, $Z_t^g(I)$ is $\cF_t$-measurable and satisfies the martingale property $\mathbb E[Z_t^g(I) \mid \cH_{t-1}] = 0$. The increments are bounded by $|Z_t^g(I)| \le 1$.The conditional variance is bounded as follows:
\begin{equation}
    \mathbb E[(Z_t^g(I))^2 \mid \cH_{t-1}] \le g(x_t) \pi_{t,I},
\end{equation}
which holds because $g(x_t) \in \{0,1\}$, $(y_t-r_I)^2 \le 1$, and the conditional variance of the indicator $\mathbf{1}[p_t=r_I]$ is $\pi_{t,I}(1-\pi_{t,I}) \le \pi_{t,I}$. Consequently, the predictable quadratic variation satisfies:
\begin{equation}
\begin{aligned}
    V_T^g(I) &:= \sum_{t=1}^T \mathbb E[(Z_t^g(I))^2 \mid \cH_{t-1}] \\
    & = \sum_{t=1}^T \mathbb{E}\left[ \bigl(\mathbf 1[p_t=r_I]-\pi_{t,I}\bigr)^2 g(x_t)^2(y_t-r_I)^2 \mid \cH_{t-1}\right]\\
    & \leq \sum_{t=1}^T \mathbb{E}\left[ \pi_{t,I}(1-\pi_{t,I}) g(x_t) \mid \cH_{t-1}\right]  \\
    &\le \sum_{t=1}^T g(x_t) \pi_{t,I} = N_T^g(I),
\end{aligned}
\end{equation}
where the first inequality is because $g(x_t)^2=g(x_t)$ and $(y_t-r_I)^2\le 1$.

\textbf{Step 2: Peeling over variance buckets.}
To handle the random nature of $N_T^g(I)$, we employ a peeling technique. Let $m_{\max} := \lceil \log_2 T \rceil$. For $m=0, 1, \dots, m_{\max}$, define the buckets:
\begin{equation}
\mathcal U_m := \{2^{m-1}<N_T^g(I)\le 2^m\},    
\end{equation}
with the convention $\mathcal U_0=\{N_T^g(I)\le 1\}$.
On $\mathcal U_m$ we have $V_T^g(I)\le 2^m$, so Lemma~\ref{lem:freedman} gives, for every $x\ge 0$,
\begin{equation}
\Pr\!\left(|M_T^g(I)|\ge x \text{ and } \mathcal U_m\right) \le 2\exp\!\left(-\frac{x^2}{2(2^m+x/3)}\right).    
\end{equation}

\textbf{Step 3: Union bound over buckets and $(g,I)$ pairs.}
Fix $\delta \in (0,1)$. For each bucket $m$, setting the threshold $x_m := \sqrt{2^{m+1} \Lambda} + \frac{2}{3} \Lambda$ where $\Lambda := \log(2(m_{\max}+1)/\delta)$ ensures the tail probability is at most $\delta / (m_{\max}+1)$. Since  $\sqrt{2^{m+1} \Lambda} \leq \sqrt{2\cdot 2^m \Lambda} \leq 2\sqrt{(N_T^g(I) \vee 1)\Lambda}$.
Summing over $m=0, \dots, m_{\max}$, we obtain that with probability at least $1-\delta$:
\begin{equation}
    |M_T^g(I)| \le 2\sqrt{(N_T^g(I) \vee 1) \Lambda} + \frac{2}{3} \Lambda.
\end{equation}
Using the inequality $\sqrt{N_T^g(I) \vee 1} \le \sqrt{N_T^g(I)} + 1$, and simplifying constants, we have:
\begin{equation}
    |M_T^g(I)| \le 4\left( \sqrt{N_T^g(I) \Lambda} + \Lambda \right).
\end{equation}
Finally, we perform a union bound over all $|\bar{\cG}| \cdot |\cD| \le 4|\bar{\cG}|\sqrt{T}$ pairs of $(g, I)$. Setting $\delta := (16 |\bar{\cG}| T^{3/2})^{-1}$ yields an event $\cE_{\rm samp}$ with $\Pr(\cE_{\rm samp}) \ge 1 - \frac{1}{4T}$. Since $\Lambda = O(\log(T |\bar{\cG}|)) = O(L)$, we conclude that on $\cE_{\rm samp}$:
\begin{equation}
    |M_T^g(I)| \le C \left( \sqrt{N_T^g(I) L} + L \right)
\end{equation}
simultaneously for all $g \in \bar{\cG}$ and $I \in \cD$, for a suitable absolute constant $C$.

\end{proof}

\subsubsection{Grouped Martingale Concentration}
\begin{lemma}[Grouped martingale concentration]
\label{lem:group-conc}
There exists an event $\cE_{\rm conc}$ of probability at least $1-\frac{1}{4T}$ such that, on
$\cE_{\rm conc}$, for every $g\in\bar{\cG}$, every interval $I\in\cD$, and every contiguous block
$S\subseteq[T]$,
\begin{equation}
\left| \sum_{t\in S} g(x_t)\pi_{t,I}(y_t-\mu_t)\right|
\le C_{\rm conc}\bigl(\sqrt{N_S^g(I)L}+L\bigr).    
\end{equation}
\end{lemma}

\begin{proof}
Fix a triple $(g, I, S)$ where $g \in \bar{\cG}$, $I \in \cD$, and $S \subseteq [T]$ is a contiguous block. The analysis proceeds in three steps, mirroring the structure of Lemma~\ref{lem:group-sampling}.

\textbf{Step 1: Martingale properties and variance bounds.}
Consider the process 
\begin{equation}
    X_t^g(I,S) := g(x_t)\pi_{t,I}(y_t-\mu_t)\mathbf{1}[t \in S].
\end{equation}
This is a two-stage martingale difference sequence with respect to $(\cH_{t-1}, \cF_t)_t$. Since $\mu_t = \mathbb{E}[y_t \mid \cH_{t-1}]$ and the weight $g(x_t)\pi_{t,I}\mathbf{1}[t \in S]$ is $\cH_{t-1}$-measurable, we have:
\begin{equation}
    \mathbb{E}[X_t^g(I,S) \mid \cH_{t-1}] = g(x_t)\pi_{t,I}\mathbf{1}[t \in S] \cdot \mathbb{E}[y_t - \mu_t \mid \cH_{t-1}] = 0.
\end{equation}
The increments are bounded by $|X_t^g(I,S)| \le 1$. The predictable quadratic variation of the sum $M_S^g(I) := \sum_{t=1}^T X_t^g(I,S)$ satisfies:
\begin{equation}
    V_S^g(I) := \sum_{t=1}^T \mathbb{E}[(X_t^g(I,S))^2 \mid \cH_{t-1}] \le \sum_{t \in S} g(x_t)\pi_{t,I} = N_S^g(I), 
\end{equation}
where we used the facts that $(y_t-\mu_t)^2 \le 1$, $g(x_t) \in \{0,1\}$, and $\pi_{t,I} \in [0,1]$.

\textbf{Step 2: Peeling over variance buckets.}
To handle the random nature of $N_S^g(I)$, we employ a peeling argument. Let $m_{\max} := \lceil \log_2 T \rceil$. For each $m=0, 1, \dots, m_{\max}$, define the variance bucket:
\begin{equation}
    \mathcal{U}_m := \{2^{m-1} < N_S^g(I) \le 2^m\}, \quad \text{with } \mathcal{U}_0 := \{N_S^g(I) \le 1\}.
\end{equation}
On the event $\mathcal{U}_m$, the quadratic variation $V_S^g(I)$ is bounded by $2^m$. Applying the two-stage Freedman inequality (Lemma~\ref{lem:freedman}), for any $x \ge 0$:
\begin{equation}
    \Pr\left( |M_S^g(I)| \ge x \text{ and } \mathcal{U}_m \right) \le 2\exp\left(-\frac{x^2}{2(2^m + x/3)}\right).
\end{equation}

\textbf{Step 3: Union bound over buckets and (g, I) pairs.}
Fix $\delta \in (0,1)$ and let $\Lambda := \log(2(m_{\max}+1)/\delta)$. Setting the threshold $x_m := \sqrt{2^{m+1}\Lambda} + \frac{2}{3}\Lambda$ and summing over all $m_{\max}+1$ buckets, we find that with probability at least $1-\delta$:
\begin{equation}
    |M_S^g(I)| \le 4 \left( \sqrt{N_S^g(I)\Lambda} + \Lambda \right).
\end{equation}
There are at most $|\bar{\cG}| \cdot |\cD| \cdot T^2 \le 4|\bar{\cG}|T^{5/2}$ such triples $(g, I, S)$. Setting $\delta := (16|\bar{\cG}|T^{7/2})^{-1}$ and performing a union bound over all triples, we obtain the event $\cE_{\rm conc}$ with probability $\Pr(\cE_{\rm conc}) \ge 1 - 1/(4T)$.

Finally, as $m_{\max} = O(\log T)$ and $\log(1/\delta) = O(\log(T|\bar{\cG}|))$, we have $\Lambda = O(L)$. Absorbing the absolute constants into $C_{\rm conc}$ yields the claimed bound.
\end{proof}

\section{Algorithmic Properties}\label{appsec:algorithm_properties}

\subsection{Weighted Feasibility (Lemma~\ref{lem:weighted-feasibility})
}\label{pf:lem:weighted}
\begin{proof}[Proof of Lemma~\ref{lem:weighted-feasibility}]
We reformulate the problem as a zero-sum game between a learner (choosing $\pi \in \Delta(B)$) and an adversary (choosing $y \in [0,1]$). Let the payoff function be defined as
\begin{equation}
F(\pi, y) := \sum_{I \in B} \pi_I \left( a_I(y - r_I) - b_I w_I \right).
\end{equation}
For any fixed $y \in [0,1]$, let $I_y \in B$ denote the unique interval containing $y$. By the definition of the grid, we have $|y - r_{I_y}| \le w_{I_y}/2$. Using the assumption $|a_{I_y}| \le b_{I_y}$, we observe
\begin{equation}
\begin{aligned}
\min_{I \in B} \left( a_I(y - r_I) - b_I w_I \right) &\le a_{I_y}(y - r_{I_y}) - b_{I_y} w_{I_y} \\
&\le |a_{I_y}| \cdot \frac{w_{I_y}}{2} - b_{I_y} w_{I_y} \\
&\le b_{I_y} \left( \frac{w_{I_y}}{2} - w_{I_y} \right) \\
& = -\frac{b_{I_y} w_{I_y}}{2} \le 0.
\end{aligned}
\end{equation}
The function $F$ is thus bilinear, and the strategy spaces $\Delta(B)$ and $[0,1]$ are compact convex sets. Thus, applying Sion's minimax theorem, we can swap the order of optimization:
\begin{equation}
\min_{\pi \in \Delta(B)} \max_{y \in [0,1]} F(\pi, y) = \max_{y \in [0,1]} \min_{\pi \in \Delta(B)} F(\pi, y) =\max_{y \in [0,1]} \min_{I \in B} \left( a_I(y - r_I) - b_I w_I \right) \leq 0
\end{equation}
Consequently, there exists a fixed $\pi \in \Delta(B)$ such that $F(\pi, y) \le 0$ for all $y \in [0,1]$. 
\end{proof}

\subsection{AdaNormalHedge Guarantee}
We now analyze the sleeping-experts wrapper from Section~\ref{sec:setup}. Concretely, we instantiate
it with the confidence-rated AdaNormalHedge algorithm of \citet{luo2015achieving}. 

\textbf{Expert construction.} 
Let $\cJ := \{(I, +), (I, -) : I \in \cD\}$ be the set of coordinates representing the positive and negative parts of the multicalibration constraints, with $M := |\cJ|$ as the total number of coordinates. Since the interval partition satisfies $|\cD| = \sum_{d=0}^{d_{\max}} 2^d \le 4\sqrt{T}$, the number of coordinates is bounded by $M \le 8\sqrt{T}$.

For every start time $s\in[T]$, group $g\in\bar{\cG}$, and coordinate $j\in\cJ$, create a
sleeping expert $e=(s,g,j)$ with prior
\begin{equation}
q_e:=\frac{1}{|\bar{\cG}|MH_{T,2}}\frac{1}{s^2},
\qquad
H_{T,2}:=\sum_{u=1}^T \frac{1}{u^2}.    
\end{equation}
Then $\sum_e q_e=1$.

\textbf{Loss for experts.}
To apply AdaNormalHedge, we map our calibration constraints into the range $[0, 1]$. Let $G = 3/2$ be a uniform bound on $|\phi_{g,I}^{\pm}(\pi_t, y_t)|$. For an expert $e = (s, g, (I, +))$, the normalized loss on an awake round is defined as:
\begin{equation}
\ell_{t,e} := \frac{G - \phi^+_{g,I}(\pi_t, y_t)}{2G}.
\end{equation}
For $e=(s,g,(I,-))$ we define the loss analogously with $\phi^-_{g,I}$ in place of
$\phi^+_{g,I}$. Let
\[
I_{t,e}:=\mathbf{1}\{s\le t,\ I\in B_t\}
\]
denote the sleeping-specialist confidence for expert $e$; in this paper these confidences are
indicators, while the factors $g(x_t)\pi_{t,I}$ enter through the expert loss itself.

\textbf{Loss for the wrapper.}
The wrapper instantiated with AdaNormalHedge maintains a set of weights $\omega_{t,e}$ over the experts $e$ that are awake at round $t$ (i.e., $I_{t,e}=1$). The wrapper loss is the weighted average $\widehat{\ell}_t := \sum_{e: I_{t,e}=1} \omega_{t,e}\ell_{t,e}$. By our normalization, this corresponds to the predicted violation $\widehat{\phi}_t=\sum_{I \in B_t} \pi_{t, I}\left(a_{t, I}\left(y_t-r_I\right)-b_{t, I} w_I\right)$ from the potential-based update:
\begin{equation}
\widehat{\ell}_t = \frac{G - \widehat{\phi}_t}{2G}.
\end{equation}
To quantify the performance relative to a specific expert $e$, we define the positive relative-loss
term up to time $b$ as:
\begin{equation}
\widetilde{L}_{b,e} := \sum_{t=1}^b I_{t,e}[\ell_{t,e} - \widehat{\ell}_t]_+.
\end{equation}
This quantity measures, up to time $b$, how much worse expert $e$ can be than the wrapper on the rounds when $e$ is awake.

\textbf{Regret guarantee.}
We now state the first-order performance guarantee of the AdaNormalHedge-based wrapper. The following bound is the confidence-rated/sleeping-specialist version of the first-order AdaNormalHedge guarantee.
\begin{proposition}[First-order confidence-rated AdaNormalHedge bound]
\label{prop:grouped-anh}
There exists an absolute constant $C_{\rm anh}>0$ such that, simultaneously for every expert $e$ and every horizon $b\le T$,
\begin{equation}
\sum_{t=1}^b I_{t,e}(\widehat \ell_t-\ell_{t,e}) \le C_{\rm anh}\left(\sqrt{\widetilde L_{b,e}L}+L\right),    
\end{equation}
where $L=\log(eT|\bar{\cG}|)$.
\end{proposition}

\begin{proof}
We combine the confidence-rated AdaNormalHedge bound of \citet[Theorem~3]{luo2015achieving}
with their first-order conversion in \citet[Theorem~2]{luo2015achieving}. The total number of
grouped sleeping experts is
\[
N=|\bar{\cG}|MT\le 8|\bar{\cG}|T^{3/2}.
\]
For expert $e=(s,g,j)$, define the instantaneous regret
\[
r_{t,e}:=I_{t,e}(\widehat\ell_t-\ell_{t,e}),\qquad
R_{b,e}:=\sum_{t=1}^b r_{t,e},
\qquad
C_{b,e}:=\sum_{t=1}^b |r_{t,e}|.
\]
The confidence-rated AdaNormalHedge guarantee gives
\[
R_{b,e}\le \sqrt{C_{b,e}A_{b,e}},
\]
where $A_{b,e}$ captures the complexity cost of expert $e$. Moreover,
\[
C_{b,e}
=
R_{b,e}+2\widetilde L_{b,e}.
\]
If $R_{b,e}<0$, the claimed upper bound is trivial. Otherwise, solving
$R_{b,e}\le \sqrt{(R_{b,e}+2\widetilde L_{b,e})A_{b,e}}$ gives
\[
R_{b,e}\le \sqrt{2\widetilde L_{b,e}A_{b,e}}+A_{b,e},
\]
which is the first-order form used below. According to \citet{luo2015achieving}, this cost is:
\begin{equation}
    A_{b,e}:=3\left(\log\frac{1}{q_e}+\log B_b+\log(1+\log N)\right),
\end{equation}
where $B_b\leq \frac{5}{2} + \frac{3}{2}\ln(1+T) =O(\log T)$. Thus, our prior yields
\begin{equation}
\log\frac{1}{q_e} = \log|\bar{\cG}|+\log M+2\log s+\log H_{T,2} 
= O(\log(T|\bar{\cG}|)) = O(L),
\end{equation}
uniformly over all experts. Since also $\log B_b=O(\log\log T)$ and
$\log(1+\log N)=O(\log\log(T|\bar{\cG}|))$, we obtain $A_{b,e}\le C'L$ for a universal constant
$C'$. Substituting into the regret inequality yields the claim.
\end{proof}

\subsubsection{Bias Control via Regret (Lemma~\ref{lem:bias-control-via-regret})
}\label{app:biascontrolviaregret}
\begin{proof}[Proof of Lemma~\ref{lem:bias-control-via-regret}]
We prove the positive direction; the negative one is identical. Let
\[
S=[a,b]\subseteq A(I),
\]
and consider the expert
\[
e=(a,g,(I,+)).
\]
Up to horizon $b$, this expert is awake exactly on the rounds in $S$. The wrapper loss satisfies
\[
\widehat \ell_t=\frac{G-\widehat \phi_t}{2G}.
\]
Using the wrapper identity from Section~\ref{sec:setup},
\[
\widehat \phi_t
=
\sum_{I'\in B_t}
\pi_{t,I'}\bigl(a_{t,I'}(y_t-r_{I'})-b_{t,I'}w_{I'}\bigr)
\le 0
\]
by the weighted-feasibility condition (\ref{eq:weighted-feas}). Therefore the regret of the wrapper
against $e$ upper bounds the cumulative positive one-sided weighted bias:
\[
\sum_{t=1}^b I_{t,e}(\widehat \ell_t-\ell_{t,e})
=
\sum_{t\in S} \frac{\phi^+_{g,I}(\pi_t,y_t)-\widehat \phi_t}{2G}
\ge
\frac{1}{2G}\sum_{t\in S}\phi^+_{g,I}(\pi_t,y_t).
\]
Proposition~\ref{prop:grouped-anh} yields
\[
\sum_{t\in S}\phi^+_{g,I}(\pi_t,y_t)
\le
C\left(\sqrt{\widetilde L_{b,e}L}+L\right).
\]
Finally, recall $N_S^g(I)=\sum_{t\in S} g(x_t)\pi_{t,I}$, and we have
\[
\widetilde L_{b,e}
\;=\;
\sum_{t\in S}\left[\frac{\widehat \phi_t-\phi^+_{g,I}(\pi_t,y_t)}{2G}\right]_+
\le
\sum_{t\in S}\frac{(-\phi^+_{g,I}(\pi_t,y_t))_+}{2G}
\le
\frac12 \sum_{t\in S} g(x_t)\pi_{t,I}
=
\frac12 N_S^g(I),
\]
because $\widehat \phi_t\le 0$ and $|\phi^+_{g,I}(\pi_t,y_t)|\le G g(x_t)\pi_{t,I}$. Rearranging gives
\[
\sum_{t\in S} g(x_t)\pi_{t,I}(y_t-r_I-w_I)
\le
C\left(\sqrt{N_S^g(I)L}+L\right).
\]
This is the point where the first-order form of Proposition~\ref{prop:grouped-anh} is essential:
although the expert is awake throughout $S$, rounds with $g(x_t)\pi_{t,I}=0$ have
$\phi^+_{g,I}(\pi_t,y_t)=0$ and do not increase $\widetilde L_{b,e}$. Thus the regret term scales
with $N_S^g(I)=\sum_{t\in S}g(x_t)\pi_{t,I}$ rather than with $|S|$ or the active lifetime of $I$.
Adding back the width term yields
\[
\sum_{t\in S} g(x_t)\pi_{t,I}(y_t-r_I)
\le
N_S^g(I)w_I+
C\left(\sqrt{N_S^g(I)L}+L\right),
\]
which is the desired positive one-sided bound. The negative direction follows from the expert
$(a,g,(I,-))$.
\end{proof}

\subsubsection{Depth-Wise Calibration Accounting (Lemma~\ref{lem:depth-calibration})
}\label{app:depthcalibration}
\begin{proof}[Proof of Lemma~\ref{lem:depth-calibration}]
Fix a group $g$ and a depth $d$, and let $C$ be the absolute universal constant whose exact values may change from line to line. For each $I\in\cV_d$, apply
Lemma~\ref{lem:bias-control-via-regret} on $A(I)$ and then use the sampling event
$\cE_{\rm samp}$:
\[
\left|
\sum_{t:p_t=r_I}g(x_t)(y_t-r_I)
\right|
\le
C\left(N_T^g(I)w_d+\sqrt{N_T^g(I)L}+L\right).
\]
By the splitting rule, every interval that ever appears satisfies
$N(I)w_d^2\le CL$, including max-depth intervals. Since $N_T^g(I)\le N(I)$, the deterministic
term is absorbed:
\[
N_T^g(I)w_d\le C\sqrt{N_T^g(I)L}.
\]
Thus
\[
\left|
\sum_{t:p_t=r_I}g(x_t)(y_t-r_I)
\right|
\le
C\left(\sqrt{N_T^g(I)L}+L\right).
\]
Summing over $I\in\cV_d$ and using Cauchy--Schwarz gives
\[
\sum_{I\in\cV_d}\sqrt{N_T^g(I)L}
\le
\sqrt{Lm_d\sum_{I\in\cV_d}N_T^g(I)}
\le
\sqrt{TLm_d},
\]
because at each round the total mass placed on depth-$d$ active intervals is at most one. If $L>T$, the
main theorem uses the trivial worst-case bound; the depth-wise bound is only invoked in the
case $L\le T$. When
$L\le T$, the deterministic active-count cap $m_d\le C T w_d^2/L$ for $d\ge1$ (see more details about this in Appendix~\ref{app:activeprofile} proof of Lemma~\ref{lem:depth-calibration}) and $m_0=1$ imply
\[
Lm_d = \sqrt{TLm_d} \cdot \sqrt{\frac{Lm_d}{T}} \leq \sqrt{TLm_d} \cdot \sqrt{Cw_d^2} \le C\sqrt{TLm_d},
\]
so the additive term is absorbed into the first branch. 
For the second bound, use $N_T^g(I)\le CL/w_d^2$ for every $I\in\cV_d$:
\[
\sum_{I\in\cV_d}\sqrt{N_T^g(I)L}
\le
\sum_{I\in\cV_d} \frac{CL}{w_d}
=
\frac{CLm_d}{w_d}.
\]
Together with $Lm_d\le Lm_d/w_d$, this proves the second branch.
\end{proof}

\section{Bounding Splits From Approximation Structure}
\subsection{Bins Far From Score Values Are Rarely Played (Lemma~\ref{lem:high-bias-rarely-played})
}\label{app:highbiasrarelyplayed}
\begin{proof}[Proof of Lemma~\ref{lem:high-bias-rarely-played}]
Fix $I$ and assume $\delta_I>C_0Bw$; otherwise only the trivial bound is claimed. Let
$S_I:=S\cap A(I)$. 
Since $\pi_{t,I}=0$ outside $A(I)$, we have $N_S(I)=N_{S_I}(I)$. If $S_I=\emptyset$, there is
nothing to prove.

Let $h_{r_I}=\sum_{g\in\bar\cG}\alpha_g g$ be the threshold representation of $f$ at threshold
$r_I$, so $\sum_g|\alpha_g|\le B$. For every $t\in S$,
\[
h_{r_I}(x_t)(f(x_t)-r_I)\ge \frac34 |f(x_t)-r_I|\ge \frac34\delta_I.
\]
Also $|h_{r_I}(x_t)|\le B$ because the groups are binary-valued. Therefore, with
\[
\rho_t:=\bigl(|\mu_t-f(x_t)|-w\bigr)_+,
\]
we have
\[
\begin{aligned}
h_{r_I}(x_t)(\mu_t-r_I)
&=
h_{r_I}(x_t)(f(x_t)-r_I)
+h_{r_I}(x_t)(\mu_t-f(x_t))\\
&\ge
\frac34\delta_I-B|\mu_t-f(x_t)|\\
&\ge
\frac34\delta_I-Bw-B\rho_t.
\end{aligned}
\]
Summing over $S_I$ gives the lower bound
\begin{equation}\label{eq:res-lower}
A_\mu
:=
\sum_{t\in S_I}\pi_{t,I}h_{r_I}(x_t)(\mu_t-r_I)
\ge
\left(\frac34\delta_I-Bw\right)N_S(I)-B\mathsf D_{S,I}.
\end{equation}

For the upper bound,
\begin{equation}\label{eq:decompau}
A_\mu = \sum_{t\in S_I}\pi_{t,I}h_{r_I}(x_t)(y_t-r_I) - \sum_{t\in S_I}\pi_{t,I}h_{r_I}(x_t)(y_t-\mu_t).   
\end{equation}
Invoking Lemma~\ref{lem:bias-control-via-regret} and then using
$N_{S_I}^g(I)\le N_S(I)$ gives
\begin{equation}\label{eq:algbias}
\begin{aligned}
\left|\sum_{t\in S_I}\pi_{t,I}h_{r_I}(x_t)(y_t-r_I)\right|
&\leq \sum_{g \in \bar{\cG}}\left|\alpha_g\right|\left|\sum_{t \in S_I} g\left(x_t\right) \pi_{t, I}\left(y_t-r_I\right)\right|\\
&\leq \sum_g\left|\alpha_g\right| N_{S_I}^g(I) w+C_{\mathrm{loc}} \sum_g\left|\alpha_g\right| \sqrt{N_{S_I}^g(I) L}+B C_{\mathrm{loc}} L \\
& \leq B N_S(I) w+C_{\mathrm{loc}} B \sqrt{N_S(I) L}+C_{\mathrm{loc}} B L,
\end{aligned}
\end{equation}
where the square-root term uses
$\sum_g|\alpha_g|\sqrt{N_{S_I}^g(I)}\le B\sqrt{N_S(I)}$. Similarly, by
$\cE_{\rm conc}$,
\begin{equation}\label{eq:martnoise}
\begin{aligned}
    \left|
    \sum_{t\in S_I}\pi_{t,I}h_{r_I}(x_t)(y_t-\mu_t)
    \right|
    &\leq \sum_{g \in \bar{\cG}}\left|\alpha_g\right|\left|\sum_{t \in S_I} g\left(x_t\right) \pi_{t, I}\left(y_t-\mu_t\right)\right|\\
    &\leq \sum_g\left|\alpha_g\right| C_{\text{conc}}\left(\sqrt{N_{S_I}^g(I) L}+L\right) \\
    &\leq C_{\text{conc}} B \sqrt{N_S(I) L}+C_{\text{conc}} B L.
\end{aligned}
\end{equation}
Substituting (\ref{eq:martnoise}) and (\ref{eq:algbias}) into (\ref{eq:decompau}) gives
\[|A_\mu|\le BN_S(I)w + CB\bigl(\sqrt{N_S(I)L}+L\bigr),
\]
where $C$ is an absolute universal constant whose exact values may change from line to line.

Combining this upper bound with (\ref{eq:res-lower}) and choosing $C_0$ large enough to absorb all
$BwN_S(I)$ terms gives, for $\Delta_I:=\delta_I-C_0Bw>0$,
\[
\Delta_I N_S(I)
\le
CB\left(\mathsf D_{S,I}+\sqrt{N_S(I)L}+L\right).
\]
Let $z:=\sqrt{N_S(I)}$. Then
\[
\Delta_I z^2\le CB\mathsf D_{S,I}+CB\sqrt{L}\,z+CBL.
\]
Young's inequality gives
$CB\sqrt L\,z\le \frac12\Delta_Iz^2+CB^2L/\Delta_I$, hence
\[
\frac12\Delta_I z^2
\le
CB\mathsf D_{S,I}+C\frac{B^2L}{\Delta_I}+CBL,
\]
and therefore
\[
N_S(I)=z^2
\le
C\left[
\frac{B^2L}{\Delta_I^2}
+
\frac{B\mathsf D_{S,I}}{\Delta_I}
+
\frac{BL}{\Delta_I}
\right].
\]
Since $\Delta_I\le1$ and $B\ge3/4$ for nonempty blocks, the last term is absorbed into
$CB^2L/\Delta_I^2$. This proves the claim.
\end{proof}

\subsection{Split Count on One Block (Lemma~\ref{lm:amortizedsplit})
}\label{app:splitcount}
\begin{proof}[Proof of Lemma~\ref{lm:amortizedsplit}]
Let $C_f(S)=\{c_1,\dots,c_K\}$ be the realized score values on $S$, and write
$R:=\mathsf R_{S,d}(f)$. Assign each depth-$d$ interval $I$ to a nearest score
value $c_{k(I)}$, and set $\delta_I:=|r_I-c_{k(I)}|$. Let $C_0$ be the constant
from Lemma~\ref{lem:high-bias-rarely-played}. Choose
\[
Q:=C_1B+\sqrt{\frac{BRw}{KL}},
\]
where $C_1$ is a sufficiently large universal constant, so that $Q\ge2C_0B$.

The near region consists of intervals with $\delta_I\le Qw$. Since depth-$d$
midpoints are spaced by $w$, each score value has $O(Q)$ nearby intervals. Hence
the near-region contribution to $\Xi_d(S)$ is $O(KQ)$.

For the far region, define distance ranges
\[
\mathcal S_{k,\ell}:=\{I:k(I)=k,\ \ell w<|r_I-c_k|\le(\ell+1)w\},
\qquad \ell\ge \lceil Q\rceil.
\]
Each range contains $O(1)$ intervals. Since $Q\ge2C_0B$, every interval in these
ranges has
\[
\Delta_I:=\delta_I-C_0Bw\ge c\ell w
\]
for a universal constant $c>0$. Lemma~\ref{lem:high-bias-rarely-played} implies
\[
\min\left\{1,\frac{N_S(I)w^2}{L}\right\}
\le
C\left[
\frac{B^2}{\ell^2}
+
\frac{B\mathsf D_{S,I}w}{\ell L}
\right].
\]
Summing the first term over all far ranges gives
\[
K\sum_{\ell\ge \lceil Q\rceil}O\!\left(\frac{B^2}{\ell^2}\right)
\le
C\frac{KB^2}{Q}
\le
CBK,
\]
using $Q\ge C_1B$. For the residual term, use $\ell\ge Q$ and
\[
\sum_{I\in\cD_d}\mathsf D_{S,I}
=
\sum_{I\in\cD_d}\sum_{t\in S\cap A(I)}
\pi_{t,I}\bigl(|\mu_t-f(x_t)|-w\bigr)_+
\le
R,
\]
because at a fixed round the total mass assigned to active depth-$d$ intervals is
at most one. Thus the far residual contribution is at most $CBRw/(QL)$.

Combining the near and far terms,
\[
\Xi_d(S)
\le
C\left[
KQ+\frac{BRw}{QL}
\right].
\]
The chosen value of $Q$ gives
\[
KQ\le CKB+C\sqrt{\frac{BKRw}{L}},
\qquad
\frac{BRw}{QL}\le C\sqrt{\frac{BKRw}{L}}.
\]
This proves the lemma.
\end{proof}

\subsection{Active Intervals From Split Bounds (Lemma~\ref{lem:active-profile})
}\label{app:activeprofile}
\begin{proof}[Proof of Lemma~\ref{lem:active-profile}]
The root is the only active depth-$0$ interval, so $m_0=1$.

For $d\ge1$, the grid has only $2^d$ depth-$d$ intervals, hence $m_d\le2^d$.
For the resource bound, each active depth-$d$ interval is created by a split of a
depth-$(d-1)$ parent. Such a parent must have accumulated at least
$L/w_{d-1}^2$ total mass before splitting. The total mass accumulated by all
depth-$(d-1)$ intervals over the horizon is at most $T$, so the number of such
parents is at most $Tw_{d-1}^2/L$, and therefore
\[
m_d\le 2T w_{d-1}^2/L=8T w_d^2/L.
\]

It remains to prove the instance-dependent bound. Every active depth-$d$ interval
is a child of a split depth-$(d-1)$ interval, so $m_d\le2|\cA_{d-1}|$. Fix any
admissible partition $\cP_{d-1}$ and triples
$\{(f_S,K_S,B_S):S\in\cP_{d-1}\}$ for \(\Psi_{d-1}^{\rm res}\). If
$I\in\cA_{d-1}$, then
$N(I)\ge L/w_{d-1}^2$. Since $\cP_{d-1}$ partitions $[T]$,
\[
N(I)=\sum_{S\in\cP_{d-1}}N_S(I),
\]
and therefore
\[
1
\le
\sum_{S\in\cP_{d-1}}
\min\left\{1,\frac{N_S(I)w_{d-1}^2}{L}\right\}.
\]
Summing over $I\in\cA_{d-1}$ and exchanging sums,
\[
|\cA_{d-1}|
\le
\sum_{S\in\cP_{d-1}}\sum_{I\in\cD_{d-1}}
\min\left\{1,\frac{N_S(I)w_{d-1}^2}{L}\right\}.
\]
Applying Lemma~\ref{lm:amortizedsplit} to each block at depth $d-1$ gives
\[
|\cA_{d-1}|
\le
C\sum_{S\in\cP_{d-1}}
\left[
B_SK_S+
\sqrt{\frac{B_SK_S\,\mathsf R_{S,d-1}(f_S)w_{d-1}}{L}}
\right].
\]
Taking the infimum over $\cP_{d-1}$ and the scores gives
$|\cA_{d-1}|\le C\Psi_{d-1}^{\rm res}(\mu_{1:T};\cG)$, and hence
$m_d\le C\Psi_{d-1}^{\rm res}(\mu_{1:T};\cG)$. Combining the three bounds proves
the lemma.
\end{proof}

\subsection{Dyadic Summation (Lemma~\ref{lem:dyadic-summation})
}\label{app:dyadicsum}
\begin{proof}[Proof of Lemma~\ref{lem:dyadic-summation}]
It is enough to prove the bound for $d\ge1$; the root contributes $O(\sqrt{TL})$,
which is absorbed by the claimed bound since $M\ge1$.

For fixed $d$, define
\[
F_d(x):=\min\left\{\sqrt{TLx},\frac{Lx}{w_d}\right\}.
\]
The function $F_d$ is nonnegative, increasing, concave on $\mathbb R_+$, and
satisfies $F_d(0)=0$. Hence it is subadditive:
\[
F_d(x+y)\le F_d(x)+F_d(y)
\qquad\text{for all }x,y\ge0.
\]
Indeed, concavity and $F_d(0)=0$ imply $F_d(\lambda z)\ge\lambda F_d(z)$ for
$\lambda\in[0,1]$; applying this with $z=x+y$ and
$\lambda=x/(x+y)$ and $\lambda=y/(x+y)$ gives the claim.

Let $C_*$ denote the constant in the hypothesis and set
\[
u_d:=C_*\min\left\{2^d,\frac{Tw_d^2}{L},M\right\},
\qquad
v_d:=C_*\min\left\{2^d,\frac{Tw_d^2}{L},A\sqrt{\frac{w_d}{L}}\right\}.
\]
Since
\[
m_d
\le
C_*\min\left\{2^d,\frac{Tw_d^2}{L},M+A\sqrt{\frac{w_d}{L}}\right\}
\le u_d+v_d,
\]
monotonicity and subadditivity give
\[
F_d(m_d)\le F_d(u_d)+F_d(v_d).
\]
Thus it suffices to bound the two sums separately.

For the $M$ part, the grid cap can only reduce the sum, so it is enough to bound
\[
\sum_d
\min\left\{
Tw_d,\sqrt{TLM},\frac{LM}{w_d}
\right\},
\]
up to constants depending only on $C_*$. The terms $Tw_d$ and $LM/w_d$ cross at
\[
w_*=(LM/T)^{1/2}.
\]
The dyadic tails on either side are geometric, and the value at the crossover is
$\sqrt{TLM}$. Endpoint cases only make the sum smaller. Therefore the $M$ part is
$O(\sqrt{TLM})$.

If $A=0$, then the $A$ part vanishes. Otherwise, the grid cap and mass cap imply
\[
F_d(v_d)
\le
C\min\left\{
Tw_d,\ \sqrt{TA}\,L^{1/4}w_d^{1/4},\ A\sqrt L\,w_d^{-1/2}
\right\}.
\]
In particular,
\[
F_d(v_d)
\le
C\min\left\{
Tw_d,\ A\sqrt L\,w_d^{-1/2}
\right\}.
\]
These two terms cross at
\[
w_*=\left(\frac{A\sqrt L}{T}\right)^{2/3},
\]
where both are equal to $T^{1/3}L^{1/3}A^{2/3}$. The dyadic tails on both sides
of this crossover are geometric. If $w_*$ lies outside the available depth range,
the nearest endpoint gives a smaller sum. Hence the $A$ part is
$O(T^{1/3}L^{1/3}A^{2/3})$.

Combining the two sums proves the claim.
\end{proof}



\section{Proofs for Tightness of the Threshold-Complexity Bound}
\label{appsec:tightness-proofs}

\subsection{Proof of Lemma~\ref{lem:generic-threshold-lb}}
\begin{proof}[Proof of Lemma~\ref{lem:generic-threshold-lb}]
The proof demonstrates that for a sufficiently fine grid, every exact threshold-simple block must
incur a complexity proportional to the number of distinct grid points it contains.

\textbf{Step 1: Distinct grid points take distinct scores.}
Fix a depth $d$ and write $w_d = 2^{-d}$. Let $c \in (0,\frac{1}{4})$ be a small absolute constant. We focus on fine-scale depths where the grid spacing dominates the discretization error:
\begin{equation}
    w_d \le \frac{c}{m} \implies |x_{i+1} - x_i| = \frac{1}{2(m-1)} > 2 w_d.
\end{equation}
Consider any $(K_{S,d}, B_{S,d}, d)$-threshold simple block $S$ and let $X(S)$ be the set of grid indices appearing in $S$. We claim that the score function $f_{S,d}$ must assign distinct values to every distinct grid point $x_i, x_j$ for $i, j \in X(S)$.

Suppose for contradiction that $f_{S,d}(x_i) = f_{S,d}(x_j)$ for $i \neq j$. By the definition of threshold simplicity, the mean $\mu$ (which in this instance equals $x$) must satisfy $|\mu - f_{S,d}| \le w_d$. Thus,
\begin{equation}
    |x_i - x_j| \le |x_i - f_{S,d}(x_i)| + |f_{S,d}(x_j) - x_j| \le 2w_d.
\end{equation}
This contradicts our choice of $w_d$, which ensures $|x_i - x_j| > 2 w_d$. Therefore, the score function must take at least $|X(S)|$ distinct values, implying $K_{S,d} \ge |X(S)|$.

\textbf{Step 2: Lower bound threshold complexity at fine-scale depths.}
Since $B_{S,d} \ge 3/4$ for any non-empty block, the cost of any admissible partition $\mathcal{P}_d$ at this depth is:
\begin{equation}
    \sum_{S \in \mathcal{P}_d} B_{S,d} K_{S,d} \ge
\frac34 \sum_{S\in\cP_d} K_{S,d} \ge \frac{3}{4} \sum_{S \in \mathcal{P}_d} |X(S)|.
\end{equation}

Because the context sequence is round-robin and $m\le T^{1/3}\le T$, every grid point $x_i$
appears at least once during the first $T$ rounds. Since the blocks in $\cP_d$ partition $[T]$, we have $\sum_{S\in\cP_d} |X(S)|\ge m$.
This yields the lower bound for a single fine-scale depth:
\begin{equation}
M_d^{\rm thr}(\mu_{1:T}; \mathcal{G}) = \inf _{\mathcal{P}_d} \sum_{S \in \mathcal{P}_d} B_{S, d} K_{S, d} \ge \frac{3m}{4}.    
\end{equation}

\textbf{Step 3: Summation over depths.}
The condition $w_d \le c/m$ is satisfied for all depths $d \ge \log_2(m/c)$. Given that 
\[
d_{\max}=\frac12\log_2 T+O(1)
\qquad\text{and}\qquad
\log_2 m\le\frac13\log_2 T,
\]
there are $\Omega(\log T)$ depths with $w_d\le c/m$. Summing the above lower bound over those
depths gives
\[
\Gamma_T^{\rm thr}(\mu_{1:T};\cG) = \sum_{d=0}^{d_{\max}} M_d^{\rm thr}(\mu_{1:T};\cG) \ge c'\,m\log T =\widetilde\Omega(m).
\]
\end{proof}

\subsection{Proof of Lemma~\ref{lem:walsh-complexity}}
\begin{proof}[Proof of Lemma~\ref{lem:walsh-complexity}]
We construct $\mathcal{G}_{T,m}$ using the subsampled Walsh system and bound its complexity in three steps.

\textbf{Step 1: Construction of $\mathcal{G}_{T,m}$.}
We define a grid index map $\operatorname{idx}: [0,1] \to [m]$ such that $\operatorname{idx}(x_i) = i$ for each grid point. Let $\psi^{\rm Wal}_\ell: \{0, \dots, m-1\} \to \{\pm 1\}$ be the standard Walsh system. The signed Walsh features are then defined as:
\begin{equation}
    w_\ell(x) := \psi^{\rm Wal}_\ell(\operatorname{idx}(x)-1) \in \{\pm 1\}.
\end{equation}
Each feature $w_\ell$ induces two binary groups $g^{\rm Wal,+}_\ell := \frac{1+w_\ell}{2}$ and $g^{\rm Wal,-}_\ell := \frac{1-w_\ell}{2}$, allowing us to write the feature as a difference $w_\ell = g^{\rm Wal,+}_\ell - g^{\rm Wal,-}_\ell$.

The subsampled Walsh family is defined as $\mathcal{G}_{T,m} := \{\mathbf{1}\} \cup \{g^{\rm Wal,+}_\ell, g^{\rm Wal,-}_\ell : \ell \in S\}$, where $S\subseteq\{1,\ldots,m-1\}$ is a subset of indices of size $|S| = O(\log^3 m)$.

\textbf{Step 2: Logarithmic representation of discrete thresholds.}
In this proof, we use the convention $\operatorname{sgn}(u)=1$ for $u\ge0$ and
$\operatorname{sgn}(u)=-1$ for $u<0$.
The key property of $\mathcal{G}_{T,m}$ is its sparse representation of thresholds. For any threshold $r$, let $k(r):=\bigl|\{i\in[m]:x_i<r\}\bigr|$ be the number of grid points below $r$. The target sign pattern on the grid, $i \mapsto \operatorname{sgn}(f^\star(x_i) - r)$, represents a simple step function: it is $-1$ for $i \le k(r)$ and $+1$ for $i > k(r)$.

By the Walsh expansion and subsampling lemmas of \citet[Section~4]{collina2026optimal},
there exists a signed linear combination of the Walsh features:
\begin{equation}
h_r(x) = \alpha_0(r)\mathbf{1}(x) + \sum_{\ell \in S} c_\ell(r) w_\ell(x),
\end{equation}
which satisfies the following two properties.

First, for every grid point $x_i$, the approximation is within the margin required for stable decision-making:
\[
\begin{cases}
h_r(x_t) \ge \frac{3}{4} & f^*(x_t)\ge r,\\ 
h_r(x_t) \le-\frac{3}{4} & f^*(x_t)<r.\end{cases}    
\]
Second, the cost of this representation, measured by the sum of absolute coefficients, is logarithmic:
\begin{equation}
|\alpha_0(r)| \le 1 \quad \text{and} \quad \sum_{\ell \in S} |c_\ell(r)| \le C \log m,
\end{equation}
for some constant $C>0$.
This implies that every threshold of $f^\star$ has representation norm $B \le 1 + 2\sum_{\ell \in S} |c_\ell| \le C \log m$ under the family $\mathcal{G}_{T,m}$.

\textbf{Step 3: Threshold complexity on the ordered grid.}
We next upper bound the threshold complexity of this Walsh instance. Fix a depth $d$ and write
$\alpha_d := 2w_d$. We define the score function $f_d$ via a monotone quantizer:
\begin{equation}
f_d(x):=\min\left\{1,\;\alpha_d\left(\left\lfloor\frac{x}{\alpha_d}\right\rfloor+\frac12\right)\right\}.    
\end{equation}
This map takes values on a grid of spacing $\alpha_d$, is nondecreasing in $x$, and satisfies $|f_d(x)-x|\le \alpha_d/2 = w_d$ for every $x\in[0,1]$.
Thus the single-block score $f_d$ has $\mathsf R_{[T],d}(f_d)=0$.
The threshold complexity $M_d^{\rm thr}$ is then bounded by two factors:
\begin{itemize}
    \item Score count: the number of values taken by $f_d$ on the realized
contexts is at most $K_d \le \min\left\{\left\lceil \frac{1}{\alpha_d}\right\rceil+1,\;m\right\}=\min\left\{ 2^d+1,\;m\right\} \le C\min\{2^d,m\}$ values for some constant $C$.
    \item Representation cost: because $f_d$ is monotone in the ordered grid index $i$, for every threshold $r$ the set $\{i\in[m]:f_d(x_i)\ge r\}$ is a suffix of the ordered grid and its complement is a prefix. Thus, every threshold of $f_d$ is again a discrete prefix/suffix sign pattern of exactly the type handled above, so the same Walsh representation gives cost $B_d\le C\log m$.
\end{itemize}
Taking the single-block partition $\cP_d=\{[T]\}$ in the definition of $M_d^{\rm thr}$ yields
\begin{equation}
M_d^{\rm thr}(\mu_{1:T};\cG_{T,m})\le C\log m\cdot \min\{2^d,m\}.
\end{equation}
Aggregating over all depths $d = 0, \dots, d_{\max}$:
\begin{equation}
\begin{aligned}
\Gamma_T^{\rm thr}(\mu_{1:T};\cG_{T,m})
&= \sum_{d=0}^{d_{\max}} M_d^{\rm thr}(\mu_{1:T};\cG_{T,m})\\
&\le C\log m \sum_{d=0}^{d_{\max}} \min\{2^d,m\}\\
&\le C m\log m\log T\\
&= \widetilde O(m).
\end{aligned}
\end{equation}
Applying Lemma~\ref{lem:generic-threshold-lb} with $\cG=\cG_{T,m}$ gives
\[
\Gamma_T^{\rm thr}(\mu_{1:T};\cG_{T,m})=\widetilde\Theta(m).
\]
\end{proof}

\subsection{Proof of Theorem~\ref{thm:collina-imported}}
\begin{proof}[Proof of Theorem~\ref{thm:collina-imported}]
The proof is the parameterized algebra underlying the Walsh lower bound of
\citet[Section~4]{collina2026optimal}. We recall the imported ingredients, translated
into our notation. Let
\[
M:=\mathbb E[\MCerr_{\cG_{T,m}}(T)],\qquad
A:=\sum_{t=1}^T |p_t-x_t|,
\qquad
\mathsf{Pred}_T:=\{p_1,\dots,p_T\},
\qquad
N:=\sum_{v\in\mathsf{Pred}_T}\sqrt{|\{t:p_t=v\}|}.
\]
For the Walsh family on a grid of size $m$, the $\ell_1$-truthfulness lemma of
\citet[Lemma~12]{collina2026optimal}, using their subsampled-Walsh approximation
lemma~\citep[Lemma~11]{collina2026optimal}, gives
\[
\mathbb E[A]\le C_1\log(m+1)\,M.
\]
Their diverse-predictions lemma~\citep[Lemma~13]{collina2026optimal} gives
\[
\mathbb E[N]\ge \frac{T}{4\sqrt{\mathbb E[A]+T/m+1}}.
\]
Finally, their adaptive-noise-bucketing theorem and its constant-group corollary,
combined with the bias--noise decomposition in
\citet[Section~4.7]{collina2026optimal}, yield, for $L_0:=\log(T+1)$,
\[
M\ge \frac{c_2}{L_0}\,\mathbb E[N]-\mathbb E[A].
\]
These imported inequalities apply for every power-of-two $m\le T^{1/3}$; the
statement in \citet{collina2026optimal} specializes them to the optimized choice
$m\asymp T^{1/3}$.

Since $\log(m+1)\le L_0$, the first and third inequalities imply
\[
M\ge \frac{c_3}{L_0^2}\,\mathbb E[N]
\]
for a universal constant $c_3>0$. Combining this with the diversity lower bound and
the truthfulness bound yields
\[
M \ge
\frac{c_3T}{4L_0^2\sqrt{C_1L_0\,M+T/m+1}}.
\]
If $C_1L_0\,M\le T/m+1$, then, since $m\le T$,
\[
M\ge c_4\,\frac{T}{L_0^2\sqrt{T/m+1}}
\ge c_5\,\frac{\sqrt{mT}}{L_0^2}.
\]
If instead $C_1L_0\,M>T/m+1$, then
\[
M\ge \frac{c_6T}{L_0^{5/2}\sqrt{M}},
\]
and hence $M\ge c_7T^{2/3}/L_0^{5/3}$. Because $m\le T^{1/3}$, this is at least
$c_7\sqrt{mT}/L_0^{5/3}$. Thus in all cases
\[
M\ge c\,\frac{\sqrt{mT}}{\log^C(T+1)}
\]
for universal constants $c,C>0$.

The equivalence with $\sqrt{T\Gamma_T^{\rm thr}}$ follows from
Lemma~\ref{lem:walsh-complexity}. Given any target complexity budget
\(\Gamma_\star\ge2\), choose $m$ to be the largest power of two such that
$m\le\min\{\Gamma_\star,T^{1/3}\}$. Then
\[
\Gamma_T^{\rm thr}(\mu_{1:T};\cG_{T,m})\le \widetilde O(\Gamma_\star),
\]
and the lower bound above gives
\[
M
\ge
\widetilde\Omega\!\left(
\min\{\sqrt{T\Gamma_\star},T^{2/3}\}
\right).
\]
\end{proof}

\end{document}